\documentclass[journal]{IEEEtran}

\usepackage[utf8]{inputenc} % allow utf-8 input
\usepackage[T1]{fontenc}    % use 8-bit T1 fonts
\usepackage[hidelinks]{hyperref}       % hyperlinks
\usepackage{url}            % simple URL typesetting
\usepackage{booktabs}       % professional-quality tables
\usepackage{amsfonts}       % blackboard math symbols
\usepackage{nicefrac}       % compact symbols for 1/2, etc.
\usepackage{microtype}      % microtypography
\usepackage{times}
\usepackage{graphicx} % more modern
\usepackage{subcaption}
\usepackage{tabularx}
\usepackage{bbm}
\usepackage{url}
\usepackage{amsmath}
\usepackage{bbm}
\usepackage{amsthm}
\usepackage{caption}
\usepackage{float}
\usepackage{amssymb}
\usepackage{grffile}
\usepackage{epstopdf}
\usepackage[colorinlistoftodos]{todonotes}
\usepackage{multirow}
\usepackage{lipsum}
\usepackage[linesnumbered,ruled,vlined]{algorithm2e}
\usepackage[inline]{enumitem}

\setlist[enumerate]{align=left,leftmargin=0mm,labelsep*=0pt}
\setlist[enumerate,1]{label=(\roman*), ref=(\roman*)}
\setlist[enumerate,2]{label=(\roman*), ref=(\roman*),leftmargin=0mm}

\newcommand{\argmin}[1]{\underset{#1}{\operatorname{argmin}}}

\newcommand{\R}{{\mathbb{R}}}
\newcommand{\E}{{\mathbb{E}}}

\newcommand{\F}{\cal{F}}

\renewcommand{\S}{\mathcal{S}}

\renewcommand{\S}{\mathcal{S}}
\newcommand{\W}{\mathcal{W}}
\newcommand{\Exp}{\mathbb{E}}

\newcommand{\var}{\operatorname{var}}
\newcommand{\Wb}{{\bar{\mathcal{W}}}}
\newcommand{\bs}{\bar{s}}
\newcommand{\cn}{\mathcal{C}}
\newcommand{\U}{\mathcal{U}}
\newcommand{\G}{\mathcal{G}}
\newcommand{\Q}{\mathcal{Q}}
\newcommand{\est}{\mathcal{E}}

\usepackage{color}

\usepackage[normalem]{ulem}

\newcommand{\hide}[1]{}
\newtheorem{theorem}{Theorem}
\newtheorem{lemma}{Lemma}
\newtheorem{proposition}{Proposition}
\newtheorem{corollary}{Corollary}
\newtheorem{claim}{Claim}

\theoremstyle{definition}
\newtheorem{remark}{Remark}

\ifCLASSOPTIONcompsoc
\usepackage[nocompress]{cite}
\else
\usepackage{cite}
\fi

\begin{document}
	\title{Order Optimal Bounds for One-Shot Federated Learning over non-Convex Loss Functions}
	
	\author{Arsalan~Sharifnassab,~%\IEEEmembership{Student~Member,~IEEE,}
		\thanks{A. Sharifnassab is with the  Computing Science Department, University of Alberta, Edmonton, Canada (e-mail: sharifna@ualberta.ca).}
		Saber~Salehkaleybar,~%\IEEEmembership{Fellow,~OSA,}
        \thanks{
			S. Salehkaleybar is with the Leiden Institute of Advanced Computer Science (LIACS), Leiden University, Leiden, Netherlands (email:s.salehkaleybar@liacs.leidenuniv.nl).}
		S.~Jamaloddin~Golestani
        \thanks{
            S. J. Golestani  is with the Department
			of Electrical Engineering, Sharif University of Technology, Tehran, 
			Iran %\protect \\
			(e-mails:golestani@sharif.edu). %Webpage: http://lis.ee.sharif.edu
		}
	}

	\markboth{One-Shot Federated Learning for Non-Convex Loss Functions}%
	{Shell \MakeLowercase{\textit{et al.}}: Non-Convex One-Shot  Federated Learning }
	
	\IEEEtitleabstractindextext{%
		\begin{abstract}
			We consider the problem of federated learning in a one-shot setting in which there are $m$ machines, each observing $n$ sample functions from an unknown distribution on non-convex loss functions. 
			Let $F:[-1,1]^d\to\R$ be the expected loss function with respect to this unknown distribution. The goal is to find an estimate of the minimizer of $F$. 
			Based on its observations, each machine generates a signal of bounded length $B$ and sends it to a server. The server collects signals of all machines and outputs an estimate of the 
			minimizer of $F$.  
			We show that  the expected loss of any algorithm is lower bounded by  $\max\big(1/(\sqrt{n}(mB)^{1/d}), 1/\sqrt{mn}\big)$, up to a logarithmic factor. 
			We then prove that this lower bound is order optimal in $m$ and $n$  by presenting a distributed learning algorithm, called Multi-Resolution Estimator for Non-Convex loss function (MRE-NC), whose expected loss matches the lower bound for large $mn$ up to  polylogarithmic factors. 
			%We propose a distributed learning algorithm, called Multi-Resolution Estimator for Non-Convex loss function (MRE-NC), whose expected error is bounded by  $\max\big(1/\sqrt{n}(mB)^{1/d}, 1/\sqrt{mn}\big)$, up to polylogarithmic factors. We also provide a matching lower bound on the performance of any algorithm, showing that MRE-NC is order optimal in terms of $n$ and $m$. 
			
			%We consider the problem of federated learning in a one-shot setting where we have $m$ machines who has access to $n$ samples from an unknown distribution. Each machine constructs a signal of limited length $B$ and sends it to a main server. The sever receives all the signals and outputs values for a model with $d$ parameters which minimizes the expected non-convex loss function. We propose a distributed learning algorithm, called Multi-Resolution Estimator for Non-Convex loss function (MRE-NC), whose expected loss function is bounded by  $(\log^3(mn)\sqrt{d})/((mB)^{1/d}\sqrt{2n})$ with respect to the optimal one. We also provide a matching lower bound on the performance of any algorithm, showing that MRE-NC is order optimal in terms of $n$ and $m$. Experiments on synthetic and real data show the effectiveness of MRE-NC in distributed learning of model's parameters for the non-convex loss functions.
		\end{abstract}
		
		\begin{IEEEkeywords}
			Federated learning, Distributed learning, Communication efficiency, non-Convex Optimization.
	\end{IEEEkeywords}}

	\maketitle

	\IEEEdisplaynontitleabstractindextext

	% \fi
	\IEEEpeerreviewmaketitle

	%\IEEEraisesectionheading{\section{Introduction}\label{sec:introduction}}
	\section{Introduction}\label{sec:introduction}
	
	\subsection{General Background}
	\IEEEPARstart{C}{onsider} a set of $m$ machines where each machine has access to $n$ samples drawn from an unknown distribution $P$. Based on its observed samples, each machine sends a single message of bounded length $B$ to a server. The server then collects messages from all machines and estimates values for model's parameters that minimize an expected loss function with respect to distribution $P$.
	
	The above one-shot setting, in which there is a single message transmission between machines and the server, is one of the scenarios in a machine learning paradigm known as ``Federated Learning''. With the advances in smart phones, these devices can collect unprecedented amount of data from interactions between users and mobile applications. This huge amount of data can be exploited to improve the performance of learned models running in smart devices. Due to the sensitive nature of the data and privacy concerns, federated learning paradigm suggests to keep users' data in the devices and train the parameters of the models by passing messages between the devices and a central server. Since mobile phones are often off-line or their connection speeds in uplink direction might be slow, it is desirable to train models with minimum number of message transmissions. 
	
	Several works have studied the problem of minimizing a convex loss function in the context of one shot distributed learning, and order optimal lower bounds and algorithms are available \cite{salehkaleybar2019one}. 
	%where the machines are allowed to send one message with a pre-specified number of bits to the server. This setting is commonly referred to as ``One-Shot Federated Learning", due to the one-shot communication between machines and the server. In the earlier work, Zhang et al. \cite{zhang2012communication} proposed averaging method in which each machine obtains the optimal values for the parameters over its own data and sends it to the server. Afterwards, the server returns the average of received parameters from machines as the output. Unfortunately, the estimation error of averaging method does not go to zero as the number of machines increases. Recently, for the convex loss functions in the setting of one-shot federated learning, Salehkaleybar et al. \cite{salehkaleybar2019one} proposed a lower bound on the estimation error achievable by any algorithm. They also proposed an order-optimal estimator whose expected error meets the mentioned lower bound up to a polylogarithmic factor. 
	However, error lower bounds in the more practical case where the loss function is non-convex, have not been well-studied. % in the literature of one-shot federated learning.
	%Although studying the case of non-convex loss function is more desirable in training deep neural networks, there is a little work in the literature of one-shot federated learning.
	
	\subsection{Our Contributions}
	In this paper, we first focus on a regime where $mB$ is large and propose a lower bound on the performance of all one-shot federated learning algorithms. We show that for sufficiently large number of machines $m$ and for any estimator $\hat\theta$, there exists a distribution $P$ and the corresponding loss function $F$ such that $\E\big[F(\hat\theta)-F(\theta^*)\big] \ge \max\big(1/(\sqrt{n} (mB)^{1/d} \ln mB), 1/\sqrt{mn}\big)$,  where $n$ is number of samples per machine, $d$ is dimension of model's parameters, $B$  is signal length in bits, and $\theta^*$ is the global minimizer of $F$.  
	Furthermore, we show that this lower bound is order optimal in  terms of $n$, $m$, and $B$.
	In particular, we propose an estimator, called Multi-Resolution Estimator for Non-Convex loss function (MRE-NC), and show that for large values of $mn$, the output $\hat\theta$ of the MRE-NC algorithm satisfies   $\E\big[F(\hat\theta)-F(\theta^*)\big] \simeq\sqrt{d}\max\big(1/(\sqrt{n}(mB)^{1/d}), 1/\sqrt{mn}\big)$. 
	We also study error-bounds under tiny communication budget, and show that if $B$ is a constant and $n=1$, the minimax error\footnote{The minimax error is defined as the smallest achievable error by an optimal estimator, considering the worst-case scenario across all possible distributions $P$.} does not go to zero even if $m$ approaches infinity and even if $d=1$. 
	
	We adopt an information-theoretic approach, focusing primarily on sample complexity rather than computational complexity, while assuming the availability of unlimited computational resources. Our results reveal a fundamental limitation of federated learning in the presence of a restricted communication budget. The lower bound in Theorem~\ref{th:lower bound} demonstrates a "curse of dimensionality" for scenarios in which $mB$ is sub-exponential in $d$. Specifically, in a centralized setting where all $nm$ data functions are accessible on the server, a computationally exhaustive algorithm can achieve a minimax error of $1/\sqrt{mn}$, significantly lower than our federated learning minimax lower bound $(1/\sqrt{n})\max(1/\sqrt{m}, 1/(mB)^{1/d})$ for a sub-exponential communication budget $B$ with respect to $d$. 
	Interestingly, the bound primarily depends on the total number of bits $mB$ received at the server. As a consequence, one can trade off the number of machines $m$ and the communication budget $B$ to maintain a fixed minimax error bound.

	%The best trade-off depends on the application. In some federated learning applications, $B$ is often limited, and $m$ is large and we do not want to force machines to send huge data.

	Another contribution of this paper is the development of novel machinery in the proof of Theorem~\ref{th:lower bound}. The machinery involves an information-theoretic lower bound for an abstract coin-flipping system (see Section~\ref{subsec:pr2}), which can be of broader  interest for the analysis of other federated learning settings, and distributed systems in general.
	
	%by reducing the problem to the problem of identifying an unfair coin among several fair coins. Using tools from information theory, we derive a lower bound on the error probability of the latter problem. In particular, we show t, indicating that MRE-NC is order optimal in terms of $n$ and $m$. 

	\subsection{Related Works}
	McMahan et al. \cite{mcmahan2017communication} considered a decentralized setting in which each machine has access to a local training data and a global model is trained by aggregating local updates. They termed this setting, ``Federated Learning'' and mentioned some of its key properties such as severe communication constraints and massively distributed data with non-i.i.d distribution. To address some of these challenges, they proposed ``FedAvg'' algorithm, which executes in several synchronous rounds. In each round, the server randomly selects a fraction of machines and sends them the current model. Each machine performs a pre-determined number of training phases over its own data.  Finally, the updated model at the server is obtained by averaging received models from the machines.
	The authors trained deep neural networks for tasks of image classification and next word prediction in a text and experimental results showed that the proposed approach can reduce the communication rounds by $10-100$ times compared with the stochastic gradient descent (SGD) algorithm. FedAvg is  generally guaranteed to converge to a first-order stationary point \cite{zhou2018convergence,wang2021cooperative}, 
	which differs from the notion of convergence to global minimum considered in this work. Moreover, the FedAvg setting is not a one-shot scenario and involves two-way communication between
	the server and the machines, resulting in  sub-optimal performance in 
	one-shot setting.

	After introducing the setting of federated learning by McMahan et al. \cite{mcmahan2017communication}, several research work addressed its challenges such as communication constraints, system heterogeneity (different computational and communication capabilities of the machines), statistical heterogeneity (data is generated in non-identically distributed manner), privacy concerns, and malicious activities. For instance, different approaches have been proposed in order to reduce the size of messages by performing quantization techniques \cite{konevcny2016federated,jiang2018linear}, updating the model from a restricted space \cite{konevcny2016federated}, or utilizing compression schemes \cite{caldas2018expanding,sattler2019robust,wu2022communication,tang20211}. To resolve system heterogeneity issues such as stragglers, asynchronous communication schemes with the assumption of bounded delay between the server and the machines have been devised \cite{zinkevich2010parallelized,ho2013more,dai2015high,nguyen2022federated,huba2022papaya}. There are also several works providing convergence guarantees for the case of non-i.i.d. samples distributed among the machines \cite{yu2018parallel,yu2019linear,li2020federated,wang2019adaptive,li2022federated,huang2022learn,li2021model,wang2020tackling,wang2020optimizing,zhu2021data}. Moreover, some notions of privacy can be preserved by utilizing differential privacy techniques \cite{abadi2016deep,bassily2014private,geyer2017differentially,papernot2018scalable,lu2019differentially,mo2021ppfl,girgis2021shuffled,gong2021ensemble,lyu2022privacy,liu2021privacy} or secure multi-party computation \cite{riazi2018chameleon,mohassel2018aby3,rouhani2018deepsecure,chen2019secure}.

	A similar setting to federated learning has been studied extensively in the literature of distributed statistical optimization/estimation with the main focus on minimizing convex loss functions with communication constraints. In this setting, 
	machines mainly reside in a data center,  
	%the number of machines are much less than the one in federated learning setting. Moreover, 
	they are much more reliable than mobile devices, and straggle nodes are less problematic. If there is no limit on the number of bits that can be sent by the machines, then each machine can send its whole data to the server. In this case, we can achieve the estimation error of a centralized solution that has access to entire data. The problem becomes non-trivial if each machine can only send a limited number of bits to the server. 
	In the one-shot setting, 
	where there is only a single one-directional message transmission from each machine to the server,
	Zhang et al. \cite{zhang2012communication} proposed a simple averaging method, in which each machine computes an estimate of optimal parameters that minimizes the empirical loss function over its own data and sends them to the server. The output of the server is the averege over the received values. For the convex functions with some additional assumptions, they showed that this method has expected error $O(1/\sqrt{mn}+1/n)$. It can be shown that this bound can be improved to $O(1/\sqrt{mn}+1/n^{1.5})$ via boot-strapping \cite{zhang2012communication} or $O(1/\sqrt{mn}+1/n^{9/4})$ by optimizing a surrogate loss function using Taylor series expansion \cite{jordan2018communication}. 
	
	Recently, for the convex loss functions in the setting of one-shot federated learning, Salehkaleybar et al. \cite{salehkaleybar2019one} proposed a lower bound on the estimation error achievable by any algorithm. They also proposed an order-optimal estimator whose expected error meets the mentioned lower bound up to a polylogarithmic factor. 
	Our bounds have three main differences with respect to \cite{salehkaleybar2019one}:  
	\begin{itemize}
 \item Here we consider general non-convex loss functions as opposed to the convex loss assumption in \cite{salehkaleybar2019one}; \item We bound $F(\theta) - F(\theta^*)$, whereas the bound in \cite{salehkaleybar2019one} is on $\|\theta - \theta^*\|$ and translates into a much weaker bound on $F(\theta) - F(\theta^*)$ compared to the results of this paper\footnote{Note that for a constant $\delta>0$ and a twice differentiable function $F$, a bound of size $\delta$ on $\|\theta - \theta^*\|$  translates into a bound of size $\delta^2$ on $F(\theta) - F(\theta^*)$, whereas in this paper we prove a much stronger (i.e., larger) lower bound of size $\delta$ on $F(\theta) - F(\theta^*)$.  More concretely,  letting $\delta=1/n^{1/2}(mB)^{1/d}$, for a convex function $F$ with bounded second derivative  (as assumed in \cite{salehkaleybar2019one}), a lower bound $\|\theta - \theta^*\| = \Omega(\delta)$ can  only imply $F(\theta) - F(\theta^*)= \Omega(\delta^2)$. In the present work, we prove a much stronger lower bound $F(\theta) - F(\theta^*)= \Omega(\delta)$. As such, the bounds in this work and \cite{salehkaleybar2019one}, despite their similarities, do not imply each other.}; and \item The proof of the present bound requires a whole new machinery that is completely different from the proof techniques used in \cite{salehkaleybar2019one}. 
		%Recently, Salehkaleybar et al. \cite{salehkaleybar2019one} proposed a lower bound on the estimation error of any algorithm and also presented an algorithm which is order-optimal. 
	\end{itemize}
 
	Zhou et al. \cite{zhou2020distilled} proposed a one-shot distillation method where each machine distills its own its data and sends the synthetic data to the server, which then trains the model  over  whole collected data. Moreover, they evaluated the proposed method experimentally on some real data, showing remarkable reduction in the communication costs. Later, Armacki et al. \cite{armacki2022one} considered clustered federated learning \cite{ghosh2022efficient} with one round of communication between machines and the server. They showed that for the strongly convex case, local computations at the machines and a convex clustering based aggregation step at the server can provide an order-optimal mean-square error rate in terms of sample complexity.

	For the case of multi-shot setting, a popular approach is based on stochastic gradient descent (SGD) in which the server queries the gradient of empirical loss function at a certain point in each iteration and the gradient vectors are aggregated by averaging to update the model's parameters \cite{bottou2010large,lian2015asynchronous,mcmahan2017communication}. In fact, FedAvg algorithm \cite{mcmahan2017communication} can be seen as an extension of SGD algorithm where each machine perform a number of training phases over its own data in each round. Although these solutions can be applied to non-convex loss functions, there is no theoretical guarantee on the quality of the output. Moreover, in the one-shot setting, the problem becomes more challenging since these gradient descent based methods cannot be adopted easily to this setting.

	\begin{figure}[t]
		\centering
		\includegraphics[width=7cm]{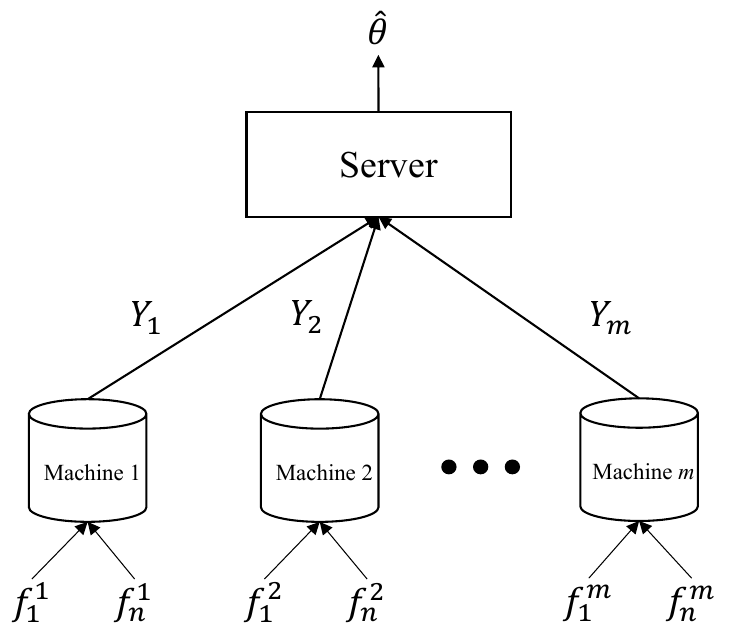}
		\caption{
			The considered distributed system consists of $m$ identical machines, each observing $n$ independent sample functions from an unknown distribution $P$. Each machine $i$ sends signal $Y_i$ of length $B$ bits to a server. The sever collects all the signals and returns an estimate $\hat{\theta}$ for the optimization problem in \eqref{eq:opt problem}.
		}
		\label{Fig:system}
	\end{figure}

	\subsection{Outline}
	The paper is organized as follows. 
	We begin with a detailed model and problem definition in Section~\ref{sec:problem def}.
	In Section \ref{sec:lower bound}, we provide a lower bound on the performance of any algorithm.
	In Section~\ref{sec:main upper convex}, we present the MRE-NC algorithm and an upper bound on its expected error that matches the previous lower bound up to polylogarithmic factors in $m$ and $n$. 
	We propose a constant lower bound on achievable error under tiny (constant) communication budget in Section~\ref{sec:tiny}.
	The proofs of our main results are presented in Sections~\ref{app:proof lowerbound} and~\ref{sec:proof main alg c}, with the details relegated to appendices for improved readability.
	Afterwards, we report some numerical experiments on small size problems in Section~\ref{sec:numerical}.
	Finally, in Section~\ref{sec:discussion}, we conclude with some remarks and open problems. 
	
	%\todo[inline]{We need to add a section for the lower bound!}

	%%%%%%%%%%%%%%%%%%%%%%%%%%%%%%%%%%%%%%%%%%%%%%%%%%%%%%%%%%%%%%%%%%%

	\medskip

	\section{Problem Definition} \label{sec:problem def}
	%We consider a model similar to \citep{saleh2019}.
	Consider a positive integer $d$ and let   $\mathcal{F}$ be the   collection of Lipschitz continuous functions over $[-1,1]^d$.
	More concretely, for any $f\in\F$ and any $\theta,\theta'\in [-1,1]^d$, we have 
	\begin{equation} \label{eq:lip}
		| f(\theta) - f(\theta')|\leq \|\theta-\theta'\|.
	\end{equation}
	Let $P$ be an unknown probability distribution over the functions in $\mathcal{F}$.  
	We define the expected loss function as follows:
	\begin{equation} \label{eq:def of the loss F}
		F(\theta) = \mathbb{E}_{f\sim P}\big[f(\theta)\big], \qquad \theta\in [-1,1]^d.
	\end{equation} 
	Our goal is to estimate a parameter $\theta^*$ that minimizes $F$:
	\begin{equation}\label{eq:opt problem}
		\theta^*=\argmin{\theta \in [-1,1]^d}\, F(\theta) .
	\end{equation}
	We assume that $\theta^*$ lies in the interior of the cube $[-1,1]^d$.
	
	%The expected loss is to be minimized in a distributed fashion, as follows.
	%We consider a distributed system comprising $m$ identical  machines and a server.
	The objective function is to be minimized in a distributed manner, as follows. The distributed systems consists of $m$ machines and a server.
	Each machine $i$ observes $n$ independently and identically distributed samples $\{f^i_1,\cdots,f^i_n\}$ drawn from the probability distribution $P$. Based on its observed samples, machine $i$ sends a signal $Y_i$ of length $B$ bits to the server.\footnote{In this context, the letter B represents the number of bits
		in each message and should not be confused with communication bandwidth, which is also
		commonly denoted by B in the communication literature.} The server collects the signals from all machines and returns an estimation of $\theta^*$, which we denote by $\hat{\theta}$. Note that in this model we consider one-way one-shot  communication between machines and the server, in the sense that each machine sends a single message to the server, while receiving no message from the server. We also assume that all machines are identical and are not enumerated in advance.   Please refer to Fig.~\ref{Fig:system} for an illustration of the distributed system.\footnote{The model of the distributed system here is similar to the one in \cite{saleh2019}.}

	\medskip

	%%%%%%%%%%%%%%%%%%%%%%%%%%%%%%%%%%%%%%%%%%%%%%%%%%%%%%%%%%%%%%%%%%%%%%%%%55

	\section{The Lower Bound}\label{sec:lower bound}
	In this section, we propose our main result, that is a lower bound on the estimation error of any algorithm. 
	We consider  a regime where $mB$ is large. In particular,
	%, that matches the upper bound of previous section up to a logarithmic factor. 
	for any constant $\cn\ge1$, given $B$ and $n$, we let $M_\cn$ be the smallest number $m$ that satisfies all of the following equations:
	\begin{equation}\label{eqa:lower condition 1} 
		\cn\sqrt{\ln(mB)} \,\ge\, 15,
	\end{equation}
	\begin{equation}\label{eqa:lower condition 2}
		mB\,\ge\, 10240,
	\end{equation}
	\begin{equation}\label{eqa:lower condition 3}
		\frac{23}{\cn \sqrt{mB}}\, +\, \frac{1}{mB} \le \frac17,
	\end{equation}
	\begin{equation}
		\begin{split}
			\frac1{B\log_2 mB}\,&\Bigg[\left(\frac{313}{\cn}\right)^2\,+\, \frac{94^2}{\cn \sqrt{mB}} \,+\, \frac{192}{mB}\\ &\quad+\frac{15}{(mB)^{1.5}} \,+\, \frac{49+6B}{(mB)^2}\Bigg] \,\le\, \frac1{10},
		\end{split}
		\label{eqa:lower condition 4}
	\end{equation}
	\begin{equation}\label{eq:mn bound}
		mn\ge 350000.
	\end{equation}
	As an example, these conditioned are satisfied for  for  $\cn=25$, $n=1$, $B=64$, and $M_\cn=4\times10^5$. 
	The following theorem presents our main lower bound.
	\begin{theorem}\label{th:lower bound}
		For any  $\cn\ge1$, any $m\ge M_\cn$, and any estimator with output denoted by $\hat\theta$, there exists a distribution $P$  and corresponding function $F$ defined in \eqref{eq:def of the loss F}, for which with probability at least $1/2$,
		\begin{equation*}
			F\big(\hat\theta\big)-F\left(\theta^*\right) \,\ge\,  \max\left(\frac{1}{\cn \sqrt{n} (mB)^{1/d} \ln mB}, \,\frac1{4\sqrt{mn}}\right).
		\end{equation*}	
	\end{theorem}
	The proof is given in Section~\ref{app:proof lowerbound}, and involves reducing the problem to the problem of identifying an unfair coin among several fair coins in a specific \emph{coin-flipping system}. % that we describe in Appendix~\ref{app:coin fliping system}. 
	We then rely on tools from information theory to derive a lower bound on the error probability of the latter problem.
	
	As an immediate corollary of Theorem~\ref{th:lower bound}, we have
	\begin{corollary}\label{cor:lower}
		For $m\ge M_\cn$,  the expected error of any estimator with output $\hat\theta$ is lower bounded by
		\begin{equation*}\label{eq:lower exp}
			\Exp\Big[ F\big(\hat\theta\big)-F(\theta^*)\Big] \ge \max\left(\frac{1}{2\cn \sqrt{n} (mB)^{1/d} \ln mB}, \,\frac1{8\sqrt{mn}}\right).
		\end{equation*}
	\end{corollary}

	For $d\ge10$, in the previous example where $\cn=25$, $B=64$, and $m\geq 4\times10^5$,
	the lower bound in %\eqref{eq:lower exp} 
	Corollary~\ref{cor:lower} would be $1/(50 \sqrt{n} (mB)^{1/d} \ln mB)$. 
	
	%\red{[Here, we can give some toy examples like: for $m=10^6$ and $B=32$ choosing $\cn=?$, the lower bound would be $???$.]}
	
	%\red{[We may also say: ``note that the lower bound in Theorem~\ref{th:lower bound} is only aimed to capture the asymptotic scaling of the error with respect  to $m$, $n$, and $B$; not characterizing the exact lower bound. Most likely, the bound can be improved to work for much smaller values of $\cn$''.] [Please add  if you think it suits.]}

	%%%%%%%%%%%%%%%%%%%%%%%%%%%%%%%%%%%%%%%%%%%%%%%%%%%%%%%%%%%%
	%%%%%%%%%%%%%%%%%%%%%%%%%%%%%%%%%%%%%%%%%%%%%%%%%%%%%%%%
	
	\medskip
	\section{Order Optimality of the Lower Bound and the MRE-NC Algorithm} \label{sec:main upper convex}
	Here, we show that the lower bound in Theorem~\ref{th:lower bound} is order optimal. We do this by proposing the MRE-NC estimator and showing that its error upper bound matches the lower bound up to polylogarithmc factors in $mn$. 
	We should however note that despite its guaranteed order optimal worst-case error bound,  benefits of applying the MRE-NC algorithm to real world problems are fairly limited. We refer the interested reader to Section~\ref{sec:discussion} for discussions on the shortcomings and scope of the MRE-NC algorithm.
	We consider general communication budget  $B\ge d\log_2 mn$.
	
	The main idea of the MRE-NC algorithm is to find an approximation of $F$ over the domain and then let $\hat\theta$ be the minimizer of this approximation.
	In order to approximate the function efficiently, transmitted signals are constructed such that the server can obtain a multi-resolution view of function $F(\cdot)$ in a grid. Thus, we call the proposed algorithm ``Multi-Resolution Estimator for Non-Convex loss (MRE-NC)". The description of MRE-NC is as follows:

	Each machine $i$ has access to $n$ functions and sends a signal $Y^i$ comprising $\lfloor B/(d\log_2 mn)\rfloor$ sub-signals of length $\lfloor d \log_2 mn\rfloor$. 
	Each sub-signal has four parts of the form $(p,\Delta,\theta^p,\eta)$. 
	The four parts $p$, $\Delta$, $\theta^p, \eta$ are as follows:
	\begin{itemize}
		\item Part $p$: Let
		\begin{equation}\label{eq:def delta c}
			\delta \,\triangleq \, \ln (mn)\, \max\left(\frac{\ln mn}{(mB)^{1/d}}, \frac{1}{m^{1/2}}\right).
		\end{equation}
		%Note that $\delta = \tilde{O}\big((mB)^{-1/d} + m^{-1/2}\big)$.
		Let $t = \log_2(1/\delta)$. %\oli{is $\lceil \cdot \rceil$ Ok?}. 
		%\red{$\delta$ not defined. $\delta = m^{-1/d}$?}
		Without loss of generality, assume that $t$ is a non-negative integer.\footnote{If $\delta>1$, we reset the value of $\delta$ to $\delta=1$. It is not difficult to check that the rest of the proof would not be upset in this spacial case.}
		%Let $C=[-1/2,1/2]^d$  be a $d$-dimensional cube with edge size one centered at $s=(0,\cdots,0)$.
		Consider a sequence of $t+1$ grids on $[-1,1]^d$ as follows.
		For  $l=0,\ldots,t$, we partition the cube $[-1,1]^d$ into $2^{ld}$ smaller equal sub-cubes with edge size $2^{-l}$. 
		The $l$th grid $G^l$ contains the centers of these smaller cubes.
		Thus, each $G^l$ has $2^{ld}$ grid points.
		%	Consider a cube , and let $\tilde{G}_s$ \oli{How about removing tilde?} be an irregular grid  with resolution $\epsilon/\sqrt{n}$, whose description follows next.
		%	The grid $\tilde{G}_s$ has a hierarchical structure and consists of $t+1$ levels $0,\ldots,t$.
		%	For any point in level $l$ of $\tilde{G}_s$
		For any point  $p'$ in $G^l$, we say that $p'$ is the parent of all $2^d$ points in $G^{l+1}$
		that are in the $2^{-l}$-cube centered at $p'$ (see Fig. \ref{Fig:MRE}). 
		Therefore, each point $G^l$ ($l<t$) has $2^d$ children. 
		
		In each sub-signal, to choose $p$, we randomly select an $l$ from $1,\dots, t$ with probability
		\begin{equation}\label{eq:prob choose p c}
			\Pr(l) = \frac{2^{(d-2)l}}{\sum_{j=1}^t 2^{(d-2)j}}.
		\end{equation}
		We then let $p$ be a uniformly chosen random grid point in $G^l$.
		Please note that the level $l$ and point $p$ selected in different sub-signals of a machine are independent and have the same distribution. 
		%Afterwards, one of the points in the grid of level $i$ is chosen uniformly. 
		
		\item Part $\Delta$: We let 
		\begin{equation}\label{eq:def hat F upper}
			F^i(\theta)\triangleq \frac2n\sum_{j=1}^{n/2} f_j^i(\theta),\quad \mbox{for } \theta\in [-1,1]^d,
		\end{equation}
		and refer to it as the empirical function of the $i$th machine.
		For each sub-signal, based on its selected $p$ part %is in $G^0$, i.e., $p=(0,\ldots,0)$, then we let $\Delta=\hat{F}^i(s)$.
		%{This will change into $\hat{F}^i(\theta)\triangleq 2/n\sum_{j=n/2+1}^n f_j^i(\theta)$.}
		%Otherwise, if $p$ is in $G^l$ for $l\geq 1$,
		we let
		\begin{equation} \label{eq:def Delta}
			\Delta \,\triangleq\, F^i (p)- F^i(p'),
			%\begin{cases} F^i (\zero) &\quad \mbox{if }  p=\zero\in G^0 \\  F^i (p)- F^i(p') & \quad \mbox{if } p\in G^l \mbox{ for } l\geq 1, \end{cases}	
		\end{equation}
		where $p'\in G^{l-1}$ is the parent of $p$. % and $\zero$ is the vector of all zeros. 
		%	Hence, the $\Delta$ can be represented by $O(d\log(mn))$ bits with a desirable accuracy.
		\item Part $\theta^p, \eta$: In the $i$th machine, if the $p$-part of a sub-signal lies in $G^t$, the machine also appends two extra pieces of information  $\theta^p, \eta$ to its sub-signal (otherwise, it sends dummy messages for these parts). 
		We let $\theta^p$ be a minimizer of  $F^i$  in the $G^t$-cube containing the point $p$, where $F^i$ is defined in \eqref{eq:def hat F upper}.
		We then set $\eta=F^i(\theta^p)-F^i(p)$. 
		%TODO: is the last sentence correct? 
	\end{itemize}
	
	\begin{figure}[t]
		\centering
		\includegraphics[width=4cm]{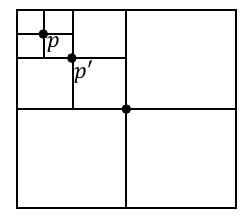}
		\caption{An illustration of a $p$-point in $[-1,1]^d$ for $d=2$. The point $p$ belongs to $G^2$ and $p'$ is the parent of $p$. }
		\label{Fig:MRE}
	\end{figure}
	
	At the server,  
	we obtain an approximation $\hat{ F}$ of the loss function $F$ over $[-1,1]^d$ as follows.
	We first eliminate redundant sub-signals so that no two surviving sub-signals from a same machine have the same $p$-parts. Hence, for each machine, the surviving sub-signals are distinct. We call this process ``redundancy elimination''. 
	%We then let $N_{s}$ be the total number of surviving sub-signals that contain $\zero=(0,\ldots,0)$ in their  $p$ part, and compute
	%\begin{equation}\label{eq:grad F0}
	%	\hat{ F}(\zero)=\frac{1}{N_{s}}\sum_{\substack{\textrm{\scriptsize{Subsignals of the form }} \\ (s,\Delta,\theta^p,\eta)\\ \mbox{\scriptsize{after redundancy elimination}}}}\Delta,
	%\end{equation}	
	We set $	\hat{ F}(0)=0$, and for any $l\geq 1$ and any  $p\in G^l$,  we let
	\begin{equation}
		\hat{F}(p)=\hat{F}(p')+\frac{1}{N_p}\sum_{\substack{\mbox{\scriptsize{Subsignals of the form }} \\ (p,\Delta,\theta^p,\eta)\\ \mbox{\scriptsize{after redundancy elimination}}}}\Delta,
		\label{eq:page17 c}
	\end{equation}
	where $N_p$ is the number of signals having point $p$ in their first argument after redundancy elimination (with the convention that $0/0=0$). After that, for each cube corresponding to a point $p$ in $G^t$, we choose a single arbitrary sub-signal of the form $(p,\Delta,\theta^p,\eta)$, from some machine $i$, and let
	\begin{equation}
		\hat{F}(\theta^p)=\hat{F}(p)+\eta=F^i(\theta^p)+\hat{F}(p)-F^i(p).
		\label{eq:Fhat}
	\end{equation} 
	Finally, the server outputs $\theta^p$ with minimum $\hat{F}(\theta^p)$.

	\begin{algorithm}[t]
		\DontPrintSemicolon
		\tcp{Constructing each sub-signal at machine $i$}
		$l \leftarrow$ choose randomly from $\{1,\cdots, t\}$ according to \eqref{eq:prob choose p c}.\;
		$p\leftarrow$ choose a point from grid $G^l$ uniformly at random.\;
		compute $\Delta$ in \eqref{eq:def Delta} for the point $p$.\;
		$\theta^p\leftarrow$ a minimizer of  $F^i$  in the $G^t$-cube centered at $p$, where $F^i$ is defined in \eqref{eq:def hat F upper}.\;
		$\eta\leftarrow F^i(\theta^p)-F^i(p)$. \;
		prepare sub-signal $(s,p,\theta^p, \eta)$ for transmission.\;
		\tcp{At the server}
		perform the process of ``redundancy elimination''.\;
		$\hat{F}(0)\leftarrow0$.\;
		\For{$l=1,\ldots,t$}{
			\For{$p\in {G}^l$}{
				compute $\hat{F}(p)$ according to \eqref{eq:page17 c}.\;	
		} }
		for each $p\in G^t$, choose an arbitrary sub-signal of the form $(p,\Delta,\theta^p,\eta)$ and compute $\hat{F}(\theta^p)$ according to \eqref{eq:Fhat}.\;
		return a $\theta^p$ with minimum $\hat{F}(\theta^p)$.\;
		\caption{MRE-NC algorithm}
		\label{Alg:MRE-NC}
	\end{algorithm}

	%In the  \MREC algorithm each sub-signals if of length $d/(d+1) \log m+d \log n$ bits, which is no larger than $d \log mn$. Therefore the the length of a signal is bounded by $\lceil B/(dlog mn)\rceil \times \lfloor d \log mn\rfloor$
	
	The following theorem provides an upper bound on the estimation error of MRE-NC algorithm, for a large-$mn$ and $B$ regime where %$m\geq \log^2(mn)$. 
	\begin{equation}\label{eq:assum1}
		\begin{split}
			&m\ge \ln^2 mn \\
			&\ln mn \ge 8\sqrt{d}\\
			& B \ge d\log(mn).
		\end{split}
	\end{equation}
	
	\begin{theorem}\label{th:main upper c}
		Consider a $d\geq 2$ and suppose that \eqref{eq:assum1} holds.
		Let $\hat{\theta}$ be the output of the MRE-NC algorithm. Then, with probability at least
		$1-\exp\big(-\Omega(\ln^2 mn)\big)$, 
		% $1-m(mn)^{d/2} \exp\left({-\ln^2 (mn)}/{128d}\right)$  we have \red{[The probability is too large. If we multiply $\delta$ and $\epsilon$ by some constants $\alpha$ and $\alpha^2$ respectively, then the exponent will be multiplied by $\alpha^2$. ($\epsilon$ is the right hand side of \eqref{eq:up}). I briefly discussed this in the Discussion section. But maybe we want to avoid the exact probability bound here, and change it to $1-\exp\big(-\Omega(\ln^2 mn)\big)$?]}
		\begin{equation} \label{eq:up}
			F(\hat{\theta})-F(\theta^*) \le 4\sqrt{d}\,\ln^2(mn)\max\left(\frac{\ln mn}{\sqrt{n}\,(mB)^{1/d}}, \,\frac{1}{\sqrt{mn}}\right).
			%\Pr\left( F(\hat{\theta})-F(\theta^*) > 4\sqrt{d}\,\ln^2(mn)\max\left(\frac{\ln mn}{\sqrt{n}\,(mB)^{1/d}}, \,\frac{1}{\sqrt{mn}}\right)\right) \,\le\, \exp\Big(-\Omega\big(\log^2(mn)\big)\Big).
		\end{equation}
	\end{theorem}

	The proof is given in  Section~\ref{sec:proof main alg c}
	%Here we present the proof sketch of Theorem~\ref{th:main upper}. The complete proof can be found in Appendix \ref{sec:proof main alg}. 
	and goes by showing that for any $l\le t$ and any $p\in G^l$, the number of received signals corresponding to $p$ is large enough so that the server obtains a good approximation of $F$ at $p$. 
	Once we have a good approximation of $F$ over $G^t$, we can find an approximate minimizer of $F$ over all $G^t$-cubes. 
	The following is  a corollary of Theorem~\ref{th:main upper c}.
	% and is proved in Appendix~\ref{proof:Cor1} \red{[Better to remove the proof. And if want to keep it, need to modify the proof. Major changes needed there. I recommend removal. This is not a difficult corollary.]}
	\begin{corollary}
		%There exist a constant $c>0$ such that for any $mn>c$, we have:
		Let $d\geq 2$ and  assume \eqref{eq:assum1}.
		Then, for any $k\ge1$,
		\begin{equation*}
			\begin{split}
				\Exp\Big[ \big|F\big(\hat\theta\big)-F(\theta^*)\big|^k\Big] \le&\max\left(\frac{4\sqrt{d}\,\ln^3 mn}{\sqrt{n}\,(mB)^{1/d}}, \,\frac{4\sqrt{d}\,\ln^2 mn}{\sqrt{mn}}\right)^k 
				\\&+ \exp\Big(-\Omega\big(\ln^2 mn\big)\Big).
			\end{split}
		\end{equation*}
		\label{corr:upper}
	\end{corollary}  
	The above upper bound matches the lower bound of Corollary~\ref{cor:lower} up to logarithmic factors with respect to $n$ and $m$, and is therefore order optimal. 
	This implies the order optimality of the MRE-NC algorithm with respect to $n$ and $m$ for the large-$mn$ regime \eqref{eq:assum1}.

	\begin{remark}
		Here, we carry out computations for the length of each subsignal.
		For the $p$ part, we need to represent the level $l$ and the point $p\in G^l$ in that level, which can be done by $\log_2 t + \log_2 2^{dt}< d \log_2 \sqrt{m}$. 
		The $\Delta$ and $\eta$ are scalars in $(-\sqrt{d}/2,\sqrt{d}/2)$ that we need to represent with precision $\epsilon/4t$, where $\epsilon$ is the expression in the right hand side of \eqref{eq:up}. Therefore, $\log_2(4t\sqrt{d}/\epsilon)< \log_2 \sqrt{mn}$ bits suffice to represent each of $\Delta$ and $\eta$. 
		Finally, $\theta^p$ is a point in a $2\delta$-cube, with a desired entry-wise precision of $\epsilon/4\sqrt{d}$. Therefore,  $\theta^p$ can be represented by $d\log_2(8\delta\sqrt{d}/\epsilon)< d\log_2 \sqrt{n}$ bits.
		Combining the above bounds, we obtain the following upper bound on the length of each subsignal: 
		$d \log_2 \sqrt{m}+ 2\log_2 \sqrt{mn} +d\log_2 \sqrt{n} = (d/2+1)\log_2 mn \le d\log_2 mn$.
	\end{remark}

	\medskip
	%%%%%%%%%%%%%%%%%%%%%%%%%%%%%%%%%%%%%%%%%%%%%%%%%%%%%%%%%%%%%%%%
	
	\section{Lower Bound under Tiny Communication Budget} \label{sec:tiny}
	The upper bound in Theorem~\ref{th:main upper c} necessitates $B \ge d\log(mn)$. In this section, we demonstrate that to make the error bound vanish for large $m$, similar to Theorem~\ref{th:main upper c}, we need $B$ to approach infinity as $m$ tends to infinity. Specifically, we examine a low-communication regime where the communication budget $B$ is constrained by a constant independent of $m$. For this regime and assuming $n=1$, we prove in the following proposition that the minimax error is lower bounded by a constant, even as $m$ approaches infinity.
	
	\begin{proposition} \label{prop:Hinf}
		Let $n=1$ and suppose that the  signal length $B$ is bounded by a constant independent of $m$.
		Then, for any estimator $\hat{\theta}$, there is a distribution $P$ over $\F$ such that  $F(\hat{\theta})-F(\theta^*)\ge   \epsilon_B$, 
		for all $m\geq 1$, where $\epsilon_B>0$ is a  constant that depends only on $B$ and is independent of $m$ and $d$.
		The above constant lower bound holds even when $d=1$.
	\end{proposition}	
	Here, we present a short proof based on Theorem~7 of \cite{salehkaleybar2019one}. Theorem~7 of \cite{salehkaleybar2019one}  establishes  existence of a distribution $P$ over strongly convex loss function with second derivatives larger than $1$ and Lipschitz constant $3$, for which an analogous constant lower bound $\|\hat{\theta}-\theta^*\|\ge \epsilon_B'$ holds, where  $ \epsilon_B'$ is a constant independent of $m$. Given the strong convexity of these loss functions, it follows that $F(\hat{\theta})-F(\theta^*)\ge (\epsilon_B')^2$. A normalization by the Lipschitz constant $3$  then implies Proposition~\ref{prop:Hinf} for $\epsilon_B  = (\epsilon_B')^2/3$.
	
	Proposition~\ref{prop:Hinf} shows that the minimax error is lower bounded by a constant regardless of $m$, when $n=1$ and $B$ is a constant.
	The constant $\epsilon_B$ in Proposition~\ref{prop:Hinf} is exponentially small in $B$. Note however that this is inevitable, because in view of Theorem~\ref{th:main upper c}, when $B = \log m$ and $n=d=1$, the error of the MRE-NC algorithm is bounded by $\tilde{O}\big(m^{-1/d}\big)$, which is exponentially small in $B$. 
	Note also that the Proposition~\ref{prop:Hinf} relies on the one-shot communication and may not hold in other settings for example when relaxing the assumption of symmetry between machines (i.e. the assumption that the machines run identical algorithms).

	\medskip
	%%%%%%%%%%%%%%%%%%%%%%%%%%%%%%%%%%%%%%%%%%%%%%%%%%%%%%%%%%%%%%%%%%%%%%%%%%%%%%%%%%%%
	
	\section{Proof of Theorem~\ref{th:lower bound}} \label{app:proof lowerbound}
	The desired lower bound is the maximum of two terms, 
	\begin{equation}\label{eq:first term}
		F(\hat\theta)-F(\theta^*)\ge 1/\big(\cn\sqrt{n}(mB)^{1/d}\ln mB\big)
	\end{equation}
	and 
	\begin{equation}
		F(\hat\theta)-F(\theta^*)\ge 1/4\sqrt{mn}.
	\end{equation}
	For the more difficult bound $1/\big(\cn\sqrt{n}(mB)^{1/d}\ln mB\big)$,
	we first  introduce a subclass  $\hat{\mathcal{F}}$ of functions in $\mathcal{F}$ and a class of probability distributions over $\hat{\mathcal{F}}$. Under these distributions, each function is generated via a process that involves flipping $mB$ coins, one of which is biased and the rest are fair. For this class, we  show that the following property holds. Once we obtain a ${1}/(2\cn \sqrt{n} (mB)^{1/d} \ln mB)$-approximate minimizer of the expected loss function $F$, we can identify the underlying biased coin. We then rely on this observation to reduce the abstract problem of identifying a biased coin among several fair coins via a certain coin flipping process to the problem of loss function minimization. We then use tools from information theory to derive a lower bound on the error probability in the former problem and conclude that the same lower bound applies to the latter problem as well.
	The second term in the lower bound, i.e. the $1/4\sqrt{mn}$ barrier, is actually well-known to hold in several centralized scenarios. 
	Here, we present a proof based  on hypothesis testing. 
	In the rest of this section, we first establish the more difficult bound $F(\hat\theta)-F(\theta^*)\ge 1/\big(\cn\sqrt{n}(mB)^{1/d}\ln mB\big)$ and introduce in Subsection~\ref{subsec:pr1} the function class $\hat{\mathcal{F}}$ and the reduction to the problem of identifying the biased coin. We then  describe the coin flipping system in more details in Subsection~\ref{subsec:pr2} and present the lower bound on the error probability in that system. Then, in Subsection~\ref{subsec:pr3}, we provide an information theoretic proof outline for this lower bound, while leaving the details until the appendices for improved readability.
	Finally, we establish the centralized bound $\Pr\big(F(\hat\theta)-F(\theta^*)\ge 1/4\sqrt{mn}\big) \ge 1/2$ in Subsection~\ref{subsec:pr0}.

	\subsection{A class of distributions}\label{subsec:pr1}
	Here, we show that $F(\hat\theta)-F(\theta^*)\ge  1/\big(\cn\sqrt{n}(mB)^{1/d}\ln mB\big)$ with probability at least $1/2$.
	For simplicity, we assume that $(mB)^{1/d}$ is an integer. 
	Consider a function $h:\R^d\to\R$ as follows. For any $\theta\in\R^n$,
	\begin{equation}\label{eqa:def h}
		h(\theta)\,=\, \begin{cases}
			(mB)^{-1/d} - \|\theta\|\quad& \textrm{if } \|\theta\| \le (mB)^{-1/d},\\
			0& \textrm{otherwise.}
		\end{cases}
	\end{equation}
	An illustration of $h(\cdot)$ is shown in Fig.\ref{fig:h f}~(a). 
	It is easy to see that $h(\cdot)$ is Lipschitz continuous with Lipschitz constant 1.
	Consider a regular grid $\G$ with edge size $2/(mB)^{1/d}$ on the cube $[-1,1]^n$.  
	We denote by $\{-1,1\}^\G$ the set of all functions from $\G$ to $\{-1,1\}$.
	To any $\sigma\in\{-1,1\}^\G$, we associate a  function $f_\sigma:\R^n\to\R$ as follows
	\begin{equation}\label{eqa:fsigma}
		f_\sigma(\theta)\,\triangleq\, \sum_{p\in\G} \sigma(p) \,h(\theta-p),\qquad \forall \theta\in\R^n.
	\end{equation}
	Fig.~\ref{fig:h f}~(b) illustrates an example of the shape of $f_\sigma$.
	Let $\hat\F$ be the set of all functions $f_\sigma$, for all $\sigma\in\{-1,1\}^\G$.
	It is easy to see that since $h(\cdot)$ is Lipschitz continuous  with Lipschitz constant $1$, each function $f_\sigma\in\hat\F$ is also Lipschitz continuous  with Lipschitz constant $1$.
	
	\begin{figure}[t!]
		\centering
		\begin{subfigure}[t]{0.3\textwidth}
			\centering
			\includegraphics[width=1\textwidth]{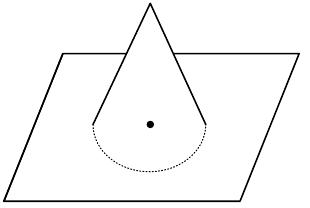}
			\hspace{2cm}
			\caption{}
		\end{subfigure}%
		\\
		\begin{subfigure}[t]{0.4\textwidth}
			\centering
			\includegraphics[width=.85\textwidth]{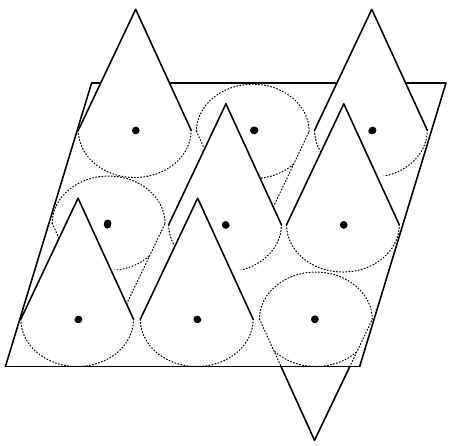}
			\caption{}
		\end{subfigure}
		\caption{Illustrations of functions $h$ and $f_\sigma$ for $d=2$. (a) shows the surface of  $h(\cdot)$ defined in \eqref{eqa:def h} and (b) is an example of  $f_\sigma(\cdot)$ defined in \eqref{eqa:fsigma}.}  \label{fig:h f}
	\end{figure}
	
	For any $p\in\G$, we define a  probability distribution $P_p$ over $\hat\F$ as follows. For any $\sigma\in\{-1,1\}^\G$,
	\begin{equation}
		P_p(f_\sigma)\, = 2^{-mB}\,\left(1-\frac{\sigma(p)}{\cn\sqrt{n} \ln  mB}\right),
	\end{equation}
	where $\cn$ is the constant in the theorem statement.
	Then, $\sum_{\sigma\in\{-1,1\}^\G} P_p(f_\sigma)=1$, and as a result, each $P_p$ is a probability distribution.
	Intuitively, when a function $f_\sigma$ is sampled from $P_p$, it is as if for every $q\in\G$ with $q\ne p$, we have $\Pr\big(\sigma(q)=1\big)=\Pr\big(\sigma(q)=-1\big)=1/2$, and for $q=p$ we have $\Pr\big(\sigma(p)=1\big)=1/2- 1/\big(2\cn \sqrt{n}\ln mB\big)$.
	This is like, the values of $\sigma(q)$ for $q\ne p$ are chosen independently at random according to the outcome of a fair coin flip, while the value of $\sigma(p)$ is the outcome of an unfair coin flip with bias $-1/\big(2\cn \sqrt{n}\ln mB\big)$, i.e., for $q\in\G$,
	\begin{equation}\label{eqa:Efsigmaq}
		\Exp_{f_\sigma\sim P_p} \big[\sigma(q)\big]\, =\, \begin{cases}
			\frac{-1}{\cn\sqrt{n} \ln mB} \quad &q=p,\\
			0&q\ne p.
		\end{cases}
	\end{equation}
	Therefore, for any $p\in\G$ and any $\theta\in[-1,1]^n$, we have
	\begin{equation}\label{eq:F h}
		\begin{split}
			F(\theta)\, &= \, \Exp_{f\sim P_p}\big(f(\theta)\big)\\
			&=\, \sum_{\sigma\in\{-1,1\}^\G} P_p(f_\sigma)\sum_{q\in\G} \sigma(q) h\big(\theta-q\big)\\
			&=\, \sum_{q\in\G}  h\big(\theta-q\big)  \sum_{\sigma\in\{-1,1\}^\G} P_p(f_\sigma) \sigma(q) \\
			&=\, \sum_{q\in\G}  h\big(\theta-q\big)  \Exp_{f_\sigma\sim P_p} \big[\sigma(q)\big] \\
			&=\,  \frac{-1}{2\cn\sqrt{n} \ln mB}\,   h\big(\theta-p\big),
		\end{split}
	\end{equation}
	where the last equality is due to \eqref{eqa:Efsigmaq}.
	Therefore,  under probability distribution $P_p$, $\theta^*=p$ is the global minimizer of $F(\cdot)$, and 
	\begin{equation*}
		F(p) \,=\, \frac{-h(0) }{\cn\sqrt{n} \ln mB} \, =\,  \frac{-1}{\cn \sqrt{n} (mB)^{1/d} \ln mB}.
	\end{equation*}
	Moreover, for any $\theta\in \R^n$ with $\|\theta-p\|\ge (mB)^{1/d}$, we have
	\begin{equation}\label{eqa:Ftheta ge Fp}
		F(\theta)\,=\,0	\,\ge\,F(\theta^*) + \frac{1}{\cn \sqrt{n} (mB)^{1/d} \ln mB},
	\end{equation}
	where $\theta^*=p$.
	
	%We now capitalize on Proposition~\ref{prop:coin flip main} in Appendix~\ref{app:coin fliping system} to complete the proof.
	We prove \eqref{eq:first term} by contradiction. Suppose that there exists an estimator $\est$, such that for any $p\in\G$, when the functions are sampled from distribution $P_p$, the estimator $\est$ returns an output $\hat\theta$, for which with probability at least $1/2$, 
	\begin{equation}\label{eqa:est  contradiction}
		F(\hat\theta)\,<\,F(\theta^*) + \frac{1}{\cn \sqrt{n} (mB)^{1/d} \ln mB}.
	\end{equation}
	Then, it follows from \eqref{eqa:Ftheta ge Fp} that $\|\hat\theta-p\|< (mB)^{1/d}$.
	In this case, $p$ is the closest grid-point of $\G$ to $\hat\theta$. 
	As a result, we can recover $p$ from $\hat\theta$, with probability at least $1/2$.
	More concretely, given estimator $\est$, we can devise an estimator $\est'$ such that for any $p\in\G$ and under distribution $P_p$, $\est'$ outputs the true $p$ with probability at least $1/2$.
	This provides a solution for the problem of identifying a biased coin among $mB-1$ unbiased coins, in a coin flipping system that we describe next.

	\subsection{Coin flipping} \label{subsec:pr2}
	Here, we describe an abstract system that aims to identify a biased coin among several  fair coins, via observing the outputs  of coin flips. We then derive a bound on the error probability of any estimator, and show that no estimator can identify the biased coin with probability at least $1/2$.

	Consider $k$ coins, one of which is biased and all others are fair. 
	The outcome of the biased coin has the following distribution:
	\begin{equation} \label{eqa:prob of biased coin}
		P(1)\,=\,\frac12 \,+\,\frac{1}{2C\sqrt{n}\ln k},\qquad P(0)\,=\,\frac12 \,-\,\frac{1}{2C\sqrt{n}\ln k}.
	\end{equation}
	% where $\ln(\cdot)$ stands for the natural logarithm.
	We index the coins by $t=1,\ldots,k$.
	The index of the biased coin is unknown initially. 
	Let $T$ denote the index of the biased coin.
	We assume that $T$ is a random variable, uniformly distributed over $1,\ldots,k$. 
	We aim to estimate $T$ by observing  outcomes of coin flips as follows.
	
	Our coin flipping system comprises $m$ machines, called the coin flippers, and a server.
	Each coin flipper flips each of every coin for $n$ times. Therefore, each coin flipper, $i$, makes a total number of $kn$ coin flips and collects the outcomes into  an $n\times k$ matrix $W^i$ with $0$ and $1$ entries.
	The $i$th coin flipper, for $i=1,\ldots,m$, then generates a $B$-bit long signal $S^i$ based on $W^i$, and sends it to the server. 
	We refer to the (possibly randomized) mapping (or coding) from $W^i$ to $S^i$ by $Q^i$. 
	The server then collects the signals of all coin flippers and generates an estimate $\hat{T}$ of the true index of the biased coin $T$.

	Let $P_e = \Pr\big(\hat{T} \ne T\big)$ be the probability  that the server fails to identify   the true biased coin index. 
	\begin{proposition} \label{prop:coin flip main}
		Let $k=mB$ and  suppose that \eqref{eqa:lower condition 1}--\eqref{eqa:lower condition 4} hold. Then, $P_e>0.5$.
	\end{proposition}
	The proof is given in the next subsection.
	The proposition asserts that no estimator can identify the biased coin with probability at least $1/2$.
	This contradicts the statement in last line of the previous subsection. 
	Hence, our initial assumption on the existence of estimator $\est$ that satisfies \eqref{eqa:est  contradiction} cannot be the case. 
	Equivalently, there exists no estimator $\est$ for which with probability at least $1/2$ we have 
	\begin{equation}\label{eq:difficult bound}
		F(\hat{\theta})\le F({\theta^*})+{1}/\big(\cn \sqrt{n} (mB)^{1/d} \ln mB\big).
	\end{equation}

	\subsection{Proof of Proposition \ref{prop:coin flip main}} \label{subsec:pr3}
	The proof relies on the following proposition.
	\begin{proposition} \label{prop:I}
		Suppose that $k=mB$ is large enough so that \eqref{eqa:lower condition 1}, \eqref{eqa:lower condition 2}, and \eqref{eqa:lower condition 3} are satisfied.
		Then, for each coin flipper, $i$, and under any coding $Q^i$, we have
		\begin{equation}\label{eqa:bound I}
			\begin{split}
				I\big(T;S^i\big) \,< \, \frac{3B}{k \ln2}\,+\,  \frac1k\,\Bigg[&\left(\frac{313}{\cn}\right)^2\,+\, \frac{94^2}{\cn \sqrt{k}} \,+\, \frac{192}k \\&\quad  \,+\frac{15}{k^{1.5}} \,+\, \frac{49+6B}{k^2}  \Bigg],
			\end{split}
		\end{equation}
		where $I(T;S^i)$ is the mutual information between $T$ and $S^i$ (see \cite{CoveT06}, page 20, for the definition of mutual information).
	\end{proposition}
	The proof of Proposition~\ref{prop:I} is pretty lengthy, and is given in Appendix~\ref{app:coin fliping system}.
	
	Given the index $T$ of the biased coin, the signals $S^1,\ldots,S^m$ will be independent. As a result, 
	\begin{equation}
		H\big(S^1,\ldots,S^m\mid T\big)\, =\, \sum_{i=1}^m H\big(S^i\mid T\big),
	\end{equation}	
	where $H(\cdot)$ is the entropy function (see \cite{CoveT06}, page 14). 
	Consequently, 
	\begin{equation}\label{eqa:subadd I}
		\begin{split}
			I\big(T;\, S^1,\ldots,S^m\big) \,&=\, H\big(S^1,\ldots,S^m\big) -H\big(S^1,\ldots,S^m\mid T\big)\\
			\,&=\, H\big(S^1,\ldots,S^m\big) -\sum_{i=1}^m H\big(S^i\mid T\big)\\
			\,&\le\, \sum_{i=1}^m H\big(S^i\big) -\sum_{i=1}^m H\big(S^i\mid T\big)\\
			\,&=\, \sum_{i=1}^m \Big(H\big(S^i\big) - H\big(S^i\mid T\big)\Big)\\
			\,&=\,\sum_{i=1}^m I\big(T; S^i\big).
		\end{split}
	\end{equation}
	Let 
	\begin{equation*}
		\epsilon\,\triangleq\, \left(\frac{313}{\cn}\right)^2\,+\, \frac{94^2}{\cn \sqrt{k}} \,+\, \frac{192}k \,+\frac{15}{k^{1.5}} \,+\, \frac{49+6B}{k^2}
	\end{equation*}
	be the expression which is a part of the right hand side of \eqref{eqa:bound I}. 
	Then, it follows from \eqref{eqa:lower condition 4}  and $k=mB$ that
	\begin{equation} \label{eqa:eps ineq}
		\frac{\epsilon}{B\log_2 k}\, \le\,\frac1{10}.
	\end{equation}
	We employ Fano's inequality (see \cite{CoveT06}, page 37), and write
	\begin{equation} \label{eqa:Pe 1/2 proof}
		\begin{split}
			P_e\,&\ge\, \frac{H\big(T\mid S^1,\ldots,S^m\big)-1}{\log_2 k}\\ 
			&= \, \frac{H(T)\,-\, I\big(T;\, S^1,\ldots,S^m\big) -1}{\log_2k}\\
			&= \, \frac{\log_2 k\,-\, I\big(T;\, S^1,\ldots,S^m\big) -1}{\log_2k}\\
			&= \, 1\,-\,\frac{ I\big(T;\, S^1,\ldots,S^m\big)}{\log_2k}\, -\, \frac1{\log_2 k}\\
			&\ge \, 1\,-\,\frac{\sum_{i=1}^m I\big(T; S^i\big)}{\log_2k}\, -\, \frac1{\log_2 k}\\
			&> \, 1\,-\,\frac{\sum_{i=1}^m \big(3B/(k\ln 2)\,+\,\epsilon/k\big)}{\log_2k}\, -\, \frac1{\log_2 k}\\
			&=\, 1\,-\, \frac{3mB}{k\ln k} \,-\, \frac{m\epsilon}{k\log_2 k} \, -\, \frac1{\log_2 k}\\
			&=\, 1\,-\, \frac{3}{\ln k} \,-\, \frac{\epsilon}{B\log_2 k} \, -\, \frac1{\log_2 k}\\
			&\ge\, 1\,-\, \frac{4}{10} \,-\, \frac{\epsilon}{B\log_2 k}\\
			&\ge\, 1\,-\, \frac{4}{10} \,-\, \frac1{10}\\
			&=\, \frac12,
		\end{split}
	\end{equation}
	where the first inequality is by the Fano's inequality, the first equality follows from the definition of mutual information,
	the second equality is because the biased coin index $T$ has uniform distribution over $1,\ldots,k$,
	the second inequality is due to \eqref{eqa:subadd I}, the third inequality follows from Proposition~\ref{prop:I}, 
	the last equality is because of the assumption $k=mB$ in the Proposition,
	the fourth inequality is due to the assumption $k=mB\ge 10240$ in \eqref{eqa:lower condition 2},
	and the last inequality is due to \eqref{eqa:eps ineq}.
	Proposition~\ref{prop:coin flip main} then follows from \eqref{eqa:Pe 1/2 proof}.

	\subsection{The centralized lower bound}\label{subsec:pr0}
	We now proceed to establish 
	$$\Pr\big(F(\hat\theta)-F(\theta^*)\ge 1/4\sqrt{mn}\big) \ge 1/2.$$ 
	Consider 9 coins, one of which is biased and all others are fair. For the biased coin, suppose that $P(1)=1/2 + 1/4\sqrt{mn}$.
	The index, $T$, of the biased coin is initially unknown. 
	We toss each of every coin for $mn$ times, and estimate an index $\hat{T}$ of the biased coin based on the observed outcomes. 
	\begin{lemma}\label{lem:9 coin}
		Assuming \eqref{eq:mn bound}, under any estimator $\hat{T}$, we have $\Pr\big(\hat{T}\ne T\big)\ge 1/2$.
	\end{lemma}
	The proof is based on the error probability of the optimal hypothesis test, and is given in Appendix~\ref{app:proof lem 9 coins}.
	In the rest of the proof,  similar to Subsection~\ref{subsec:pr1}, we consider a collection of functions and a probability distribution over them, such that for the corresponding expected loss function $F$, finding a $\hat\theta$ with $F(\hat\theta)\,<\,F(\theta^*) + 1/4\sqrt{mn}$ leads to the identification of a biased coin in the setting of Lemma~\ref{lem:9 coin} with probability at least $1/2$. This is a contradiction, and establishes the nonexistence of such estimator.
	Since the argument is very similar to the line of arguments in Subsection~\ref{subsec:pr1}, here we simply state the main result in the form of a lemma and defer the detailed proof until Appendix~\ref{app:proof centralized bound}.
	\begin{lemma}\label{lem:centralized}
		Assuming \eqref{eq:mn bound}, for any estimator $\hat{T}$, there exists a distribution under which $\Pr\big(F(\hat\theta)-F(\theta^*)\ge 1/4\sqrt{mn}\big) \ge 1/2$.
	\end{lemma}
	Finally,  Theorem~\ref{th:lower bound} follows from \eqref{eq:difficult bound} and Lemma~\ref{lem:centralized}.

	%%%%%%%%%%%%%%%%%%%%%%%%%%%%%%%%%%%%%%%%%%%%%%%%%%%%%%%%%%%%%%%%%%%%%%%%%%
	%%%%%%%%%%%%%%%%%%%%%%%%%%%%%%%%%%%%%%%%%%%%%%%%%%%%%%%%%%%%%

	\medskip
	\section{Proof of Theorem~\ref{th:main upper c}} \label{sec:proof main alg c}
	In this proof, we adopt several ideas from the proof of Theorem~4  in \cite{salehkaleybar2019one}.
	For simplicity and without loss of generality, throughout this proof,  we assume that for any $f\in \mathcal{F}$,
	\begin{equation}
		f(0)=0.
	\end{equation}
	This is without loss of generality because adding to each function $f\in \mathcal{F}$ a constant $-f(0)$ does not change the estimation $\hat\theta$.
	We first show that for $l=1,\ldots, t$ and for any $p\in G^l$, the number of sub-signals corresponding to $p$ after redundancy elimination  is large enough so that the server obtains a good 
	approximation of $F$ at $p$. 
	Once we have a good approximation of $F$ at all points of $G^t$, we can find an approximate minimizer of $F$.
	Let 
	\begin{equation} \label{eq:def eps upper}
		\begin{split}
			\epsilon \,&\triangleq\, \frac{4\delta \sqrt{d}\ln(mn)}{\sqrt{n}} \\\, &=\, 4\sqrt{d} \,\ln^2(mn) \, \max\left( \frac{\ln mn }{(mB)^{1/d}\sqrt{n}},\, \frac1{\sqrt{mn}} \right). 
		\end{split}
	\end{equation}
	For any $p\in \bigcup_{l\le t} G^l$, let $N_p$ be the number of machines that select point $p$ in at least one of their sub-signals. Equivalently,  $N_p$ is the number of sub-signals after redundancy elimination that have point $p$ as their second argument.
	Let $\mathcal{E}$ be the event that for $l=1,\ldots, t$ and for any $p\in G^l$, we have 
	\begin{equation} \label{eq:bound Np}
		N_p \ge   \frac{2\ln^4(mn)\, 2^{-2l}}{n\epsilon^2} . %\frac{1}{4d}2^{-2l} \, \log^{2d-6}(mn)\, (mB)^{2/d}.
	\end{equation}
	Then,
	\begin{lemma} \label{lem:prob of E2 upper}
		$\Pr\big(\mathcal{E} \big) \ge 1- m^{d/2} \exp\big(-\ln^2(mn)/32d\big)$.
	\end{lemma}
	The proof is based on the concentration inequality in Lemma~\ref{lemma:CI}~(b), and is given in Appendix~\ref{app:proof lem E2 upper}.
	
	%In the remainder of the proof, we assume that event $E_2$ has been also occurred.
	Capitalizing on Lemma~\ref{lem:prob of E2 upper}, we now obtain a bound on the estimation error of $F$ over $G^l$. 
	Let $\mathcal{E'}$ be the event that for $l=1,\ldots, t$ and any grid point $p\in G^l$, we have 
	\begin{equation} \label{eq:prob E3 upper}
		\big|\hat{F} (p) - F(p)\big|\,<\,\frac{\epsilon}{8}.
	\end{equation}
	
	\begin{lemma} \label{lem:prob E3 upper}
		$\Pr\big(\mathcal{E'}\big)\,\ge\,1-m^{d/2} \exp\big(-\ln^2(mn)/32d\big)-2m^{d/2} \exp\big(-\ln^2(mn)/128d\big)$. 
	\end{lemma}
	The proof is given in Appendix~\ref{app:proof E3 upper} and relies on Hoeffding's inequality and the lower  bound on the number of received signals for each grid point, driven in Lemma~\ref{lem:prob of E2 upper}. 
	For each  $p\in G^t$, let $cell_p$ be the small cube with edge size $2\delta$ that is centered at $p$. 
	Let $\mathcal{E''}$ be the event that for any machine $i$, any $p\in G^t$, and any $\theta\in cell_p$, %and any subsignal of the form $(p,\Delta,\theta^p,\eta)$, we have:
	\begin{equation} \label{eq:prob E4 upper}
		\Big|\big(F^i(\theta)-F^i(p)\big)-\big(F(\theta)-F(p)\big)\Big|<\frac{\epsilon}{8}.
	\end{equation}
	\begin{lemma} \label{lem:prob E4 upper}
		$\Pr\big(\mathcal{E''}\big)\,\ge\,1-2  n^{d/2} m^{1+d/2}\exp\big(-\ln^2(mn)/64\big)$. 
	\end{lemma}
	
	The proof is given in Appendix \ref{app:proof E4 upper}.
	Assuming $\mathcal{E''}$, it follows from \eqref{eq:Fhat} and \eqref{eq:prob E4 upper} that for any $p\in G^t$, and for the subsignal $(p,\Delta,\theta^p,\eta)$ that is used in the computation of $\hat{F}(\theta^p)$ in \eqref{eq:Fhat}, we have
	\begin{equation}\label{eq:Fhat e''}
		\big|\big(\hat{F}(\theta^p) - \hat{F}(p)\big) - \big( F(\theta^p)-F(p)  \big)\big| \le \frac\epsilon8.
	\end{equation}
	
	%In the remainder of the proof, we assume that   $\mathcal{E'}$ and $\mathcal{E''}$ hold. Before proceeding, we need the following lemma.
	The following auxiliary lemma has a straightforward proof.
	\begin{lemma}
		Consider a $\gamma>0$ and a function $g$ over a domain $\mathcal{W}$. 
		Let  $\hat{g}$ be a uniform $\gamma$-approximation of $g$, that is $|\hat{g}(w)-g(w)|\leq\gamma$, for all $w\in \mathcal{W}$.
		Let $w^*$ be the minimizer of $\hat{g}$ over $\mathcal{W}$. Then,  $g(w^*)\leq \inf_{w\in \mathcal{W}} g(w) +2\gamma$.
		\label{lemma:aux}
	\end{lemma}

	%Assume that $\mathcal{E''}$ holds. 
	Consider a point $p\in G^t$ and the subsignal $(p,\Delta,\theta^p,\eta)$ that is used in the computation of $\hat{F}(\theta^p)$ in \eqref{eq:Fhat}. Suppose that this subsignal has been generated in the $i$th machine. 
	Let $\hat{g}(\theta)=F^i(\theta)-F^i(p)$, $g(\theta)=F(\theta)-F(p)$, and $\mathcal{W}=cell_p$. 
	Assuming $\mathcal{E''}$, $\hat{g}$ is an $\epsilon/8$-approximation of $g$, and Lemma \ref{lemma:aux} implies that $
	g(\theta^p)\leq g(\theta_{cell_p}^*)+{\epsilon}/{4}$, where $\theta_{cell_p}^*$ is the minimizer of $F$ in $cell_p$. Therefore,
	\begin{equation} \label{eq:lemma G}
		\begin{split}
			F(\theta^p)\leq F(\theta_{cell_p}^*)+\frac{\epsilon}{4}.
		\end{split}
	\end{equation}
	Moreover, assuming $\mathcal{E'}$ and $\mathcal{E''}$, we obtain
	\begin{equation}
		\begin{split} \label{eq: Fi upper}
			|\hat{F}(\theta^p)&-F(\theta_{cell_p}^*)|=\big|\big(\hat{F}(\theta^p) - \hat{F}(p)\big) - \big( F(\theta^p)-F(p)  \big)\\
			&\qquad +\big(\hat{F}(p)-F(p)\big) + \big(F(\theta^p)-F(\theta_{cell_p}^*)\big)\big|\\
			&\le \big|\big(\hat{F}(\theta^p) - \hat{F}(p)\big) - \big( F(\theta^p)-F(p)  \big)\big|\\
			&\qquad +\big|\hat{F}(p)-F(p)\big| + \big|F(\theta^p)-F(\theta_{cell_p}^*)\big|\\
			&\le \frac\epsilon8 + \frac\epsilon8 + \frac\epsilon4\\
			&= \frac\epsilon2
		\end{split}
	\end{equation}
	where the last inequality follows from \eqref{eq:Fhat e''}, \eqref{eq:prob E3 upper}, and \eqref{eq:lemma G}
	By further assuming $\mathcal{E}$, we know that each cell is selected by at least one machine. Then, applying Lemma~\ref{lemma:aux} on \eqref{eq: Fi upper} with identifications  $\mathcal{W}=\{\theta^P: p\in G^t\}$, $\hat{g}(\theta^p)=\hat{F}(\theta^p)$ and ${g}(\theta^p)=F(\theta_{cell_p}^*)$, we obtain
	\begin{equation}
		F(\hat\theta) \,\le\, \min_{p\in G^t} F(\theta_{cell_p}^*) + \epsilon \,=\,  F(\theta^*) + \epsilon.
	\end{equation}
	Substituting the probabilities of events  $\mathcal{E}$, $\mathcal{E'}$, and $\mathcal{E''}$ from
	lemmas \ref{lem:prob of E2 upper}, \ref{lem:prob E3 upper}, and \ref{lem:prob E4 upper}, respectively,  it follows that $F(\hat{\theta})\leq F(\theta^*)+\epsilon$ with probability at least 
	\begin{equation*}
		\begin{split}
			&1- \big(1-\Pr(\mathcal{E})\big)- \big(1-\Pr(\mathcal{E}')\big)- \big(1-\Pr(\mathcal{E}'')\big) \\
			&\ge 1- 2m^{d/2} \Bigg[\exp\left(\frac{-\ln^2 mn}{32d}\right)+ \exp\left(\frac{-\ln^2 mn}{128d}\right)  \\
			&\qquad\qquad\qquad\,\,+ n^{d/2} m\exp\left(\frac{-\ln^2 mn}{64}\right)\Bigg]\\
			&\ge  1-m(mn)^{d/2} \exp\left(\frac{-\ln^2 mn}{128d}\right).
		\end{split}
	\end{equation*}
	This completes the proof of Theorem~\ref{th:main upper c}.

	%%%%%%%%%%%%%%%%%%%%%%%%%%%%%%%%%%%%%%%%%%%%%%%%%%%%%
	
	\medskip
	%%%%%%%%%%%%%%%%%%%%%%%%%%%%%%%%%%%%%%%%%%%%%%%%
	
	\section{Experiments} \label{sec:numerical}
	Here we study performance of the MRE-NC algorithm on problems of small sizes. 
	Note that when $d$ is large, the lower bound $1/\sqrt{n}(mB)^{1/d}$ in Theorem~\ref{th:lower bound}, scales poorly with respect to $mB$. 
	This eliminates the hope for efficient and  guaranteed loss minimization in large problems, and  limits the applicability of the MRE-NC algorithm to  problems with large dimensions. In this view, in this section we focus on small size problems and demonstrate  performance of MRE-NC on small toy examples.
	\vspace{-.1cm}
	\subsection{Synthetic Data}
	We evaluated the performance of MRE-NC and compared it with two naive approaches: 1- the averaging method from \cite{zhang2012communication}: each machine obtains empirical loss minimizer on its own data and sends to the server. The output would be the average of received signals at the server side. 2- Single machine method: similar to the previous method, each machine sends the empirical loss minimizer to the server. At the server, one of the received signals is picked randomly and returned as the output.
	\begin{figure}[t]
		\centering
		\includegraphics[width=8.5cm]{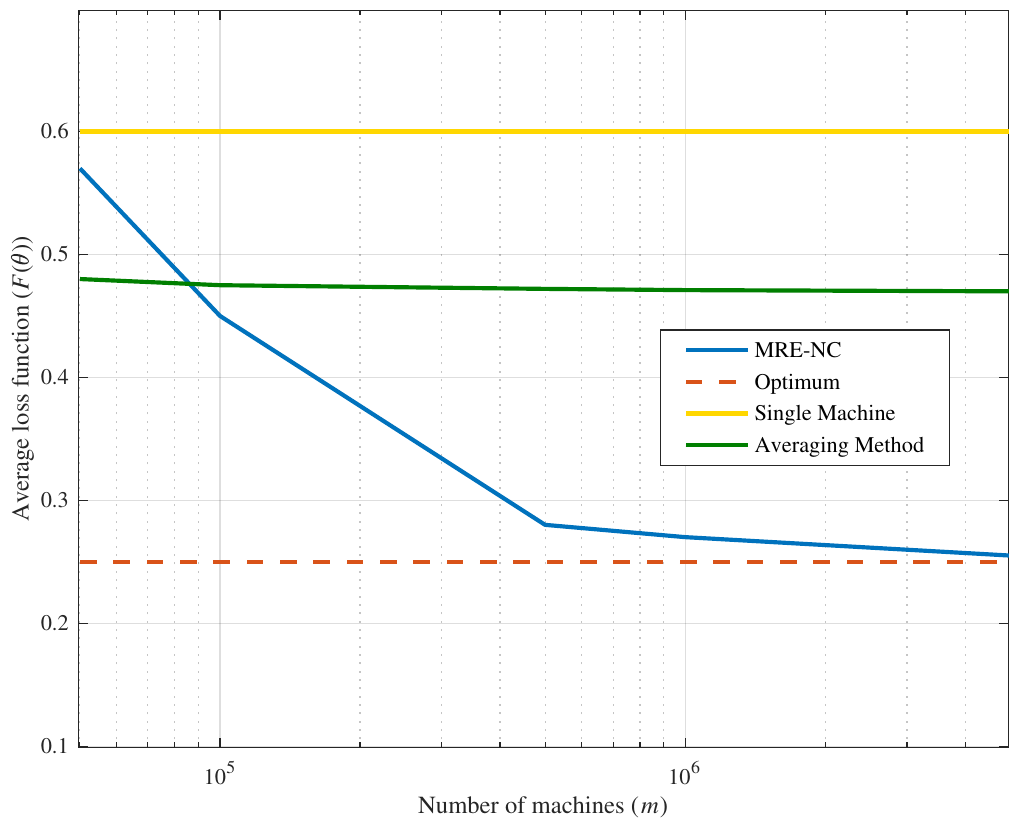}
		\caption{Comparison of the performance of MRE-NC with two naive approaches. The number of parameters ($d$) and the number of samples per machine ($n$) are $6$ and  $10$, respectively.}
		\label{Fig:MRE:Sim}
	\end{figure}
	
	In our experiment, each sample $(x,y)$, $x\in \mathbb{R}^2$, and $y\in \R$ is generated according to $y=\theta_2^T ReLU(\theta_1 x)+N$ where $ReLU(x)=\max(0,x)$ is the rectified linear unit, and the entries $[\theta_1]_{2\times 2}$ are drawn from a uniform distribution in the range $[-2,2]$ and $[\theta_2]_{2\times 1}=[1,-1]$. Moreover, $N$ is sampled from Gaussian distribution $\mathcal{N}(0,0.5)$. We considered the mean square error as the loss function.
	
	In Fig. \ref{Fig:MRE:Sim}, the value of $F(\theta)$ is depicted versus number of machines for MRE-NC and two naive approaches. In this experiment, we assumed that each machine has access to $n=10$ samples. As can be seen, the MRE-NC algorithm outperforms the two naive methods, its performance improves as the number of machines increases, and approaches to the optimal value. 
	
	\subsection{Real Data}
	
	In this part, we apply the MRE-NC algorithm to the task of classifying images of digits in the MNIST dataset \cite{lecun1998mnist}. We employed an ensemble learning technique \cite{shalev2014understanding} to build a model at the server side.
	In ensemble learning, we train a set of models, commonly called weak learners, that perform slightly better than random guess. Afterwards, a strong model can be built based on the models through different techniques such as boosting, stacking, or even picking the weak learner with the best performance.	
	In this experiment, we obtained a collection of weak learners  by running multiple instances of MRE-NC algorithm in parallel and then selected the one which has the lowest estimated empirical loss. More specifically, we assumed that each machine has access to $n=10$ random samples from MNIST dataset. Furthermore, for each image $X\in \mathbb{R}^{28\times 28}$, residing in each machine, that machine splits $X$ horizontally or vertically at pixel $p\in \{7, 14, 21\}$ into two parts, computes the average values of pixels in each part, and finally scales these average values into the range $[0,100]$. Let $(Z_1^{h,p},Z_2^{h,p})$ and $(Z_1^{v,p},Z_2^{v,p})$ be the resulted values for the horizontal or vertical split at pixel $p$, respectively. We considered the model $h(Z)=sigmoid(\theta_1^TZ+\theta_2)$, where $sigmoid(x)=1/(1+\exp(-x))$,  $\theta_1\in\mathbb{R}^2$, $\theta_2\in \mathbb{R}$, and $Z$ is the sample obtained after pre-processing at the machine as described above for any horizontal or vertical split. We considered the cross-entropy loss function (see page 72 in \cite{goodfellow2016deep}) and trained six models by executing MRE-NC algorithm on the data obtained from each horizontal/vertical split at pixel $p\in \{7, 14, 21\}$. At the server side, the model with the minimum $\hat{F}$ was selected. In our experiments, we considered images of only two digits $3$ and $4$ and tried to classify them\footnote{We considered a binary classification problem in our experiment, and any pair of digits with different shapes can be chosen for the considered task. Herein, we picked the two digits 3 and 4 that are different in shape, and the six weak learners have a wide range of performance in terms of accuracy on these two digits.}.  Fig. \ref{Fig:MRE:Real}, depicts the true $F(\hat{\theta})$ and the error in classification averaged over $10$ instances of the problem. As can be seen, both metrics decrease as the number of machines increases. Moreover, these metrics approach the optimal values corresponding to the centralized solution in which the server has access to the entire data. 
	
	\begin{figure}[t]
		\centering
		\includegraphics[width=8.8cm]{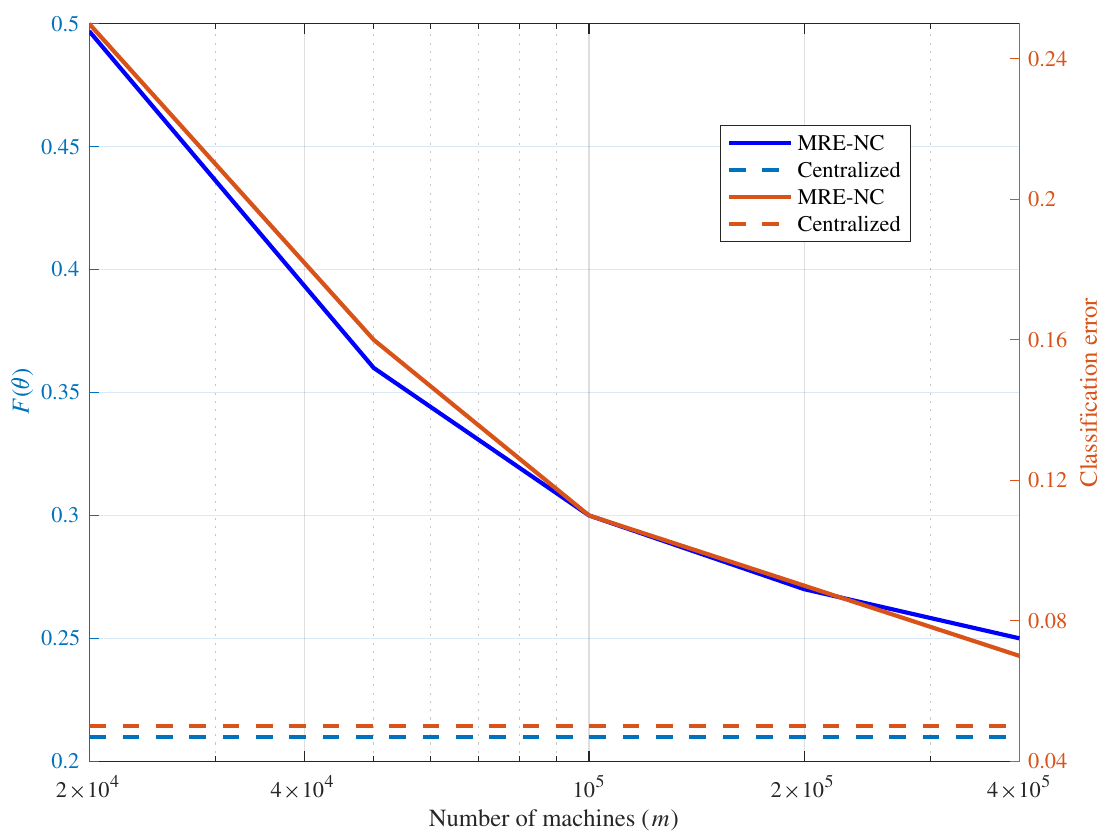}
		\caption{The performance of MRE-NC (loss and classification error) on classifying digits in MNIST dataset against the number of machines. The left and right y-axes correspond to the true loss function and the classification error of the obtained model, respectively. The number of parameters per weak learner and the number of samples per machine ($n$) are $3$ and  $10$, respectively.}
		\label{Fig:MRE:Real}
	\end{figure}

	%%%%%%%%%%%%%%%%%%%%%%%%%%%%%%%%%%%%%%%%%%%%%%%%
	
	\medskip
	\section{Discussions} \label{sec:discussion}
	In this paper, we studied the problem of federated learning in a one-shot setting and under limited communication budget $B$.
	We presented a general lower bound and showed that, ignoring logarithmic factors, the expected loss $\E\big[F(\hat\theta)-F(\theta^*)  \big]$ of any estimator is lower bounded by $\max\left(n^{-1/2}\,(mB)^{-1/d}, \,(mn)^{-1/2}\right)$.
	We then proposed an estimator called  MRE-NC, whose expected loss matches the above lower bound, and is therefore optimal.  
	We also established a constant lower bound on minimax error when the communication budget is constrained by a constant.
	The class of functions we considered in this paper is pretty general. We do not assume differentiability and our class includes all Lipschitz continuous functions over $[-1,1]^d$. This makes the model suitable for use in modern machine learning settings such as neural networks.
	
	The MRE-NC algorithm works by finding an $O(1/\sqrt{n}(mB)^{1/d})$-approximation of the value of the expected loss function $F$ over a fine grid of size $mB$.
	To do this, the algorithm adopts a multi-resolution idea from the MRE-C algorithm \cite{salehkaleybar2019one} which was previously proposed for the case of convex loss functions. The overall structure and the details of the MRE-NC algorithm are however different from those in \cite{salehkaleybar2019one}. While our upper bound proof incorporates several ideas from the upper bound proof in \cite{salehkaleybar2019one}, the proof of our lower bound is novel and relies on reductions from the problem of identifying an unbiased coin in a certain coin flipping system. 
	The proof involves  information theory, and despite the simple appearance of the coin flipping problem, it has not been studied previously, to the best of our knowledge.

	Our lower bound implies that the worst case expected error of no estimator can decrease faster than roughly $1/\sqrt{n}(mB)^{1/d}$. When $d$ is large, the error bound scales poorly with respect to $mB$. 
	%This suggests that increasing the number of samples per machine ($n$) can result in faster error reduction, compered to increasing the number of machines ($m$) or the communication budget ($B$).
	This eliminates the hope for efficient and  guaranteed loss minimization in large problems, and  limits the applicability of the MRE-NC algorithm to the problems with large dimensions. 
	On the positive side, as we demonstrated in the numerical experiments, the MRE-NC algorithm can be effectively employed to solve small size problems. Moreover, for large dimensional problems, when incorporated into an ensemble learning system, it proves effective for training weak learners (refer to  Section~\ref{sec:numerical} for further discussions).

	A drawback of the MRE-NC algorithms is that each machine requires to know $m$ in order to set the number of levels for the grids. This can be circumvented by considering infinite number of levels, and letting the probability that $p$ is chosen from level $l$ decrease exponentially with $l$. 
	As another drawback of the MRE-NC algorithms, note that each machine $i$ needs to compute the minimizer $\theta^p$ of its local function $F^i$ in a small cube around the corresponding point $p$. Since $F^i$ is a non-convex function, finding $\theta^p$ is in general computationally exhaustive. Although this will not affect our theoretical bounds, it would further limit the applicability of MRE-NC algorithm in practice.
	Moreover, it is good to point a possible trade off between the coefficients in the precision and  probability exponent of our bounds. More specifically, if we multiply the upper bound in Theorem~\ref{th:main upper c} by a constant, then the corresponding probability exponent will be multiplied by the square of the same constant. In this way, one can obtain smaller upper bounds for larger values of $mn$.

	For future works,
	given the poor scaling of the lower bound in terms of $m$ and $d$, it would be important to devise 
	scalable heuristics that are practically efficient in one shot learning system classes of interest, like neural networks.
	Moreover, efficient accurate solutions might be possible under  further assumptions on the class of functions and distributions.
	On the theory side, the bounds in this paper are minimax bounds. From a practical perspective, it is important to develop average case bounds under reasonable assumptions. 
	Another interesting direction is to relax the assumption of fixed $n$ number of  samples per machine, and to prove lower and upper bounds if the $i$th machine receives $n_i$ samples.

	%\ifCLASSOPTIONcompsoc
	%\section*{Acknowledgments}
	%\else
	%\section*{Acknowledgment}
	%\fi
	%This research was  supported by INSF under contract No. 97012846.
	
	\bibliographystyle{IEEEtran}
	\bibliography{ref}

	% or
	

	\newpage
	\onecolumn
	\appendices
	%%%%%%%%%%%%%%%%%%%%%%%%%%%%%%%%%%%%%%%%%%%%%%%%%%%%%%%%%%%%%%%%%%%%%
	\section{Concentration inequalities} \label{app:prelim}
	Here,
	we collect two well-known concentration inequalities that will be used in the proofs of our main results.
	\begin{lemma} (Concentration inequalities)
		\begin{enumerate}
			\item[(a)] (Hoeffding's inequality)  
			Let $X_1,\cdots,X_n$ be independent random variables ranging over the interval $[a,a+\gamma]$. Let $\bar{X}=\sum_{i=1}^n X_i/n$ and $\mu =\mathbb{E}[\bar{X}]$.
			%uppose that the domain of $X_i$, $i=1,\ldots, n$, is included in . 
			Then, for any $\alpha>0$,
			\begin{equation*}
				\Pr\big(|\bar{X}-\mu|>\alpha\big)\leq 2\exp\left(\frac{-2n\alpha^2}{\gamma^2}\right).
			\end{equation*}
			
			\item[(b)] (Theorem 4.2 in \cite{motwani1995randomized})  
			Let $X_1,\cdots,X_n$ be independent Bernoulli  random variables, ${X}=\sum_{i=1}^n X_i$, and $\mu =\mathbb{E}[{X}]$.
			Then, for any $\alpha\in(0,1]$,
			\begin{equation*}
				\Pr\big(X<(1-\alpha)\mu\big)\leq \exp\left(-\frac{\mu\alpha^2}{2}\right).
			\end{equation*}
		\end{enumerate}
		\label{lemma:CI}
	\end{lemma}

	%%%%%%%%%%%%%%%%%%%%%%%%%%%%%%%%%%%%%%%%%%%%%%%%%%%%%%%5
	%%%%%%%%%%%%%%%%%%%%%%%%%%%%%%%%%%%%%%%%%%%%%%%%%%%%%%%%%%%%%%
	%%%%%%%%%%%%%%%%%%%%%%%%%%%%%%%%%%%%%%%%%%%%%%%%%%%%%%%%%%%%%%%%

	%%%%%%%%%%%%%%%%%%%%%%%%%%%%%
	%\medskip
	%
	%\section{Proof of Corollary \ref{corr:upper}}\label{proof:Cor1}
	%\red{[REMOVE THIS APPENDIX]}
	%From Theorem \ref{th:main upper c}, we have:
	%\begin{equation}
	%\begin{split}
	%	\mathbb{E}[F(\hat{\theta})-F(\theta^*)]&\leq
	%	 (1-\exp\big(-\Omega\big(\log^2(mn)\big)\big)\frac{\log^3(mn)\sqrt{d}}{(mB)^{1/d}\sqrt{2n}}+\exp\big(-\Omega\big(\log^2(mn)\big)\big)\max_{\theta^*,\hat{\theta}} (F(\hat{\theta})-F(\theta^*))\\
	%	 &\leq\frac{\log^3(mn)\sqrt{d}}{(mB)^{1/d}\sqrt{2n}}+\exp\big(-\Omega\big(\log^2(mn)\big)2\sqrt{d},
	%\end{split}
	%\end{equation}
	%where the second inequality is due to Assumption \eqref{eq:lip}, i.e., $|F(\hat{\theta})-F(\theta^*)|\leq \|\hat{\theta}-\theta^*\|\leq 2\sqrt{d}$. Hence, one can obtain some constant $C$ such that for any $mn>C$, the sum in right hand side of last inequality is less than $\frac{\log^3(mn)\sqrt{d}}{(mB)^{1/d}\sqrt{n}}$ and the proof is complete.

	%%%%%%%%%%%%%%%%%%%%%%%%%%%%%%%%%%%%%%%%%%%%%%%%%%%%%%%%%%%%%%%%
	
	%%%%%%%%%%%%%%%%%%%%%%%%%%%%%
	\medskip

		\section{Proof of Proposition~\ref{prop:I}} \label{app:coin fliping system}
		The proof comprises a series of lemmas whose proofs are given in the form of separate subsections at the end of this appendix, for improved readability.
		For simplicity of the notation, we drop the coin-flipper's index from all equations, and will write $S$, $W$, and $Q$ in places of $S^i$, $W^i$, and $Q^i$, respectively.
		Recall that for a coin flipper, $W$ is the $n\times k$ binary matrix of its coin flip outcomes, so that the $j$th column of $W$ corresponds to the outcomes of the $j$th coin, for $j=1,\ldots,k$. 
		We refer to the (possibly randomized) mapping (or coding) from $W$ to the $B$-bit signal $S$ by $Q$. 
		More concretely, $Q\big(S\mid W\big)$ denotes the probability that a machine outputs signal $S$ given the coin flipping outcomes $W$, for all $W\in \W$ and all $S\in \S$, where $\W$ is the set of all $n\times k$ matrices with $0$ and $1$ entries, and $\S$ is the set of all $B$-bit signals.
		We begin by showing that the mutual information $I(T;S)$ is maximized via a coding $Q$ that is  deterministic.
		We call a coding $Q$ deterministic if $Q(S\mid W)$ is either $0$ or $1$, for all $W\in\W$ and $S\in\S$.
		
		\begin{lemma} \label{lem:deterministic coding}
			Among all randomized codings $Q$, there exists a deterministic coding that maximizes $I(T;S)$.
		\end{lemma}
		The proof relies on a well-known result on the convexity of mutual information with respect to $\Pr(S\mid T)$, and is given in Appendix~\ref{app:proof lem deterministic coding}.
		
		In light of Lemma~\ref{lem:deterministic coding}, for the rest of the proof without loss of generality we assume that the coding $Q$ is deterministic. 
		Equivalently, corresponding to each $s\in\S$, we associate a subset of $\W$ whose elements are mapped to $s$. 
		With an abuse of notation, we denote this subset of $\W$ by $s$. 
		In other words, to any $s\in\S$, is associated a subset $s\subseteq \W$ containing all $w\in\W$ that are mapped to $s$ via the deterministic coding.
		For any $w\in\W$ and for $t=1,\ldots,k$, we denote the $t$th column of $w$ by $w_t$.
		Given $w\in\W$,  we let $P(w)$ be the probability that $w$ is the outcome matrix of coin-flips when $T$ is chosen uniformly at random from $1,\ldots,k$. Moreover, given $t\le k$, we let  $P(w\mid T=t)$ be the probability that $w$ is the outcome matrix of coin-flips when $T=t$.
		%For the rest of the proof, we fix the universal constant 
		%\begin{equation}
		%	\beta\,\triangleq \, \Exp\big[ P(w) \big]\, =\, \sum_{w\in\W} \big(P(w)\big)^2.
		%\end{equation}
		
		\begin{lemma} \label{lem:pw bound}
			There exists % a universal constant $\beta\ge0.46\times2^{-kn}$ and 
			a subset $\Wb\subseteq \W$ with $\Pr(\Wb)\ge 1-6k^{-3}$, such that for any $w\in\Wb$,
			\begin{equation}\label{eqa:lemI pwt}
				\Pr\big(W_t=w_t\mid T=t\big)\le \frac{5\times 2^{-n}}3 ,\qquad \textrm{ for } \,t=1,\ldots,k,
			\end{equation}
			and
			\begin{equation}\label{eqa:lemI pw}
				%\big|P(w)-\beta \big|\, \le 2^{-nk} \, \frac{23}{\cn \sqrt{k}}.
				2^{-kn}\,(1-\delta)\,\le\, P(w)\,\le\, 2^{-kn}\,(1+\delta),
			\end{equation}
			where 
			\begin{equation} \label{eqa:def del}
				\delta\,\triangleq\, \frac{23}{\cn \sqrt{k}}\, +\, \frac1k.
			\end{equation}
		\end{lemma}
		The proof relies on concentration inequalities, and is presented in Appendix~\ref{app:proof lem pw bound}.
		For the rest of this appendix, we fix the constant $\delta$ and the set $\Wb$ as defined in Lemma~\ref{lem:pw bound}.
		%It follows from \eqref{eqa:lemI pw} and $\beta\ge0.46\times2^{-kn}$ that $\big|P(w)-\beta \big|\, \le \, {50\beta}/{\cn \sqrt{k}}$. Equivalently, for any $w\in\Wb$,
		%\begin{equation}\label{eqa:Pw beta 1}
		%	\left(1-\frac{50}{\cn \sqrt{k}}\right)\beta\,\le\, P(w)\,\le\,\left(1+\frac{50}{\cn \sqrt{k}}\right)\beta.
		%\end{equation}
		%Then, \eqref{eqa:Pw beta 1} simplifies to
		%\begin{equation}\label{eqa:Pw beta}
		%	(1-\delta)\beta\,\le\, P(w)\,\le\,(1-\delta)\beta.
		%\end{equation}
		
		Recall the convention that for any $s\in\S$, we denote the subset of $\W$ that is mapped to $s$, also by $s$. 
		For the simplicity of notation, for the rest of the proof, for any $s\in \S$, we let $\bs\triangleq s\bigcap \Wb$ and $P(\bs\mid T=t)\triangleq P(w\in \bs\mid T=t)$.
		We make the convention that $0/0=1$.
		
		\begin{lemma}\label{lem:HI}
			\begin{itemize}
				\item[a)] The entropy $H(S)$ of signal $S$ satisfies
				\begin{equation}
					H(S) \,\ge \, \left(\sum_{s\in\S} P(\bs)\log_2\frac{1}{P(\bs)}\right)\, -\,\frac{ 9}{k^{3}}.
				\end{equation}
				
				\item[b)] The mutual information $I(T;S)$ satisfies
				\begin{equation}
					I(T;S) \,\le\, \frac1{k \ln 2} \sum_{s\in\S} P(\bs) \sum_{t=1}^k \left( \frac{P(\bs\mid T=t)}{P(\bs)} \,-1 \right)^2 \quad+\, \frac{9+6B + 15\sqrt{k}}{k^3}.
				\end{equation}
			\end{itemize}  
		\end{lemma}
		The proof  is given in Appendix~\ref{app:proof lem HI}.
		Our next lemma provides a bound on the weighted sum of a probability mass function in terms of its entropy.
		
		\begin{lemma}\label{lem:alpha}
			Consider an integer $n\ge1$ and a set $\big\{\alpha_u\mid u\in\{0,1\}^n\big\}$ of real numbers such that $\alpha_u\in [-1,1]$, for all $u\in\{0,1\}^n$, and $\sum_{u\in\{0,1\}^n}\alpha_u=0$. 
			Let $U$ be a random variable on $\{0,1\}^n$ with probability distribution $P$, such that for any $u\in \{0,1\}^n$, we have $U=u$ with probability $P(u)$. Then,
			\begin{equation}\label{eqa:alpha 1.5}
				\left( \sum_{u\in\{0,1\}^n} \alpha_u P(u)\right)^2\, \le\, 1.5\big(n-H(U)\big),
			\end{equation}
			where $H(U)$ is the entropy of $U$. 
		\end{lemma}
		The proof is presented in Appendix~\ref{app:proof lem alpha}. 
		We now have all the required lemmas, and are ready to prove Proposition~\ref{prop:I}.
		
		For any $t\le k$, any $u\in \{0,1\}^n$, and any $s\in\S$, let $N^{\bs}_t(u)$ be the number of $w\in\bs$ such that $w_t=u$. Also, let
		%\begin{equation} \label{eqa:def q}
		%	q_{t}^{\bs}(u) \,\triangleq\, \frac{N^{\bs}_t(u)}{|\bs|},
		%\end{equation}
		%where 
		$|\bs|$ be the size of the set $\bs$. %Note that $\sum_{u\in \{0,1\}^n} q_{t}^{\bs}(u) =1$, and as a result, $q_{t}^{\bs}(\cdot)$ is a probability distribution over $ \{0,1\}^n$.
		%On the other hand, 
		Then, for any $u\in \{0,1\}^n$, 
		\begin{equation}\label{eqa:prop2 *1}
			\begin{split}
				\Pr\big(W_t=u\mid W\in \bs\big) \, &=\, \frac{\sum_{\substack{w\in\bs \\ w_t=u}} P(w)}{\sum_{w\in\bs} P(w)}\\
				&\le\,\frac{\sum_{\substack{w\in\bs \\ w_t=u}} (1+\delta)2^{-kn}}{\sum_{w\in\bs} (1-\delta)2^{-kn}}\\
				& =\, \frac{1+\delta}{1-\delta}\,\frac{N^{\bs}_t(u)}{|\bs|},
				%&=\,  \frac{1+\delta}{1-\delta}\,q_{t}^{\bs}(u),
			\end{split}
		\end{equation}
		where the inequality follows from \eqref{eqa:lemI pw}. % and the last equality is from the definition of $q_{t}^{\bs}(u)$ in \eqref{eqa:def q}.
		In the same vein,  for any $u\in \{0,1\}^n$, 
		\begin{equation} \label{eqa:prop2 *2}
			\Pr\big(W_t=u\mid W\in \bs\big) \,\ge \, \frac{1-\delta}{1+\delta}\,\frac{N^{\bs}_t(u)}{|\bs|}.%   q_{t}^{\bs}(u).
		\end{equation}
		
		Let 
		\begin{equation}\label{eqa:def U 53}
			\U\,\triangleq \left\{u \in \{0,1\}^n\Big| \quad \Pr\big(W_1=u\mid T=1\big)\le\frac{5\times 2^{-n}}3\right\}.
		\end{equation}
		It follows from \eqref{eqa:lemI pwt} that 
		\begin{equation}\label{eqa:Wb then U}
			\textrm{if}\quad w\in\Wb, \quad \textrm{then}\quad w_t\in\U, \quad \mbox{for} \quad t=1,\ldots,k.
		\end{equation}
		Therefore, for any $u\in\{0,1\}^n$, any $t\le k$,  and any $s\in\S$,
		\begin{equation}\label{eqa:notinU then Nu=0}
			\textrm{if}\quad u\not\in\U, \quad \textrm{then}\quad N_t^{\bs}(u) =0.
		\end{equation}
		Let
		\begin{equation}\label{eqa:def gamma}
			\gamma\triangleq \sum_{u\in \U} \Pr\big(W_1=u\mid T=1\big).
		\end{equation}
		Then, for the random outcome matrix $W$ of the coin flipping, we have
		\begin{equation}	\label{eqa:gamma 56}
			\gamma 
			\,=\, \Pr\big(W_1\in\U\mid T=1\big)
			\,\ge\, \Pr\big(W\in \Wb \mid T=1\big)
			\,=\, \Pr\big(W\in \Wb\big)
			\,=\, P(\Wb) \,\ge\, 1-6k^{-3} \,\ge\, \frac56,
		\end{equation}
		where the first inequality follows from \eqref{eqa:Wb then U}, the first equality is due to the symmetry and invariance of the set $\Wb$ with respect to permutation of different columns, the second inequality is due to Lemma~\ref{lem:pw bound}, and the third inequality is because $6k^{-3}\le 1/6$ (see~\eqref{eqa:lower condition 2} with identification $k=mB$). 
		For any $u\in\{0,1\}^n$, let
		\begin{equation}\label{eqa:def alpha u}
			\alpha_u\,\triangleq\,\begin{cases}
				\frac{2^n}{\gamma} \Pr\big(W_1=u\mid T=1\big)-1\quad& \textrm{if } u\in\U, \\ -1& \textrm{if } u\not\in\U .
			\end{cases}
		\end{equation}
		It follows from \eqref{eqa:gamma 56} and the definition of $\U$ in \eqref{eqa:def U 53} that for any $u\in\U$,  we have $2^n P(W_1=u\mid T=1)/\gamma \le 2^n P(W_1=u\mid T=1) \times 6/5\le 2$. Therefore, $\alpha_u\in[-1,1]$, for all $u\in\{0,1\}^n$. 
		Moreover,
		\begin{equation*}
			\sum_{u\in\{0,1\}^n} \alpha_u \,=\, -2^n+\frac{2^n}{\gamma}\sum_{u\in\U} \Pr\big(W_1=u\mid T=1\big)
			\,=\, -2^n + \frac{2^n}{\gamma} \times\gamma
			\,=\, 0,
		\end{equation*}
		where the second equality is from the definition of $\gamma$ in \eqref{eqa:def gamma}.
		Hence, the set of numbers $\alpha_u$, for $u\in\{0,1\}^n$, satisfies all of the conditions in Lemma~\ref{lem:alpha}. Therefore, it follows from Lemma~\ref{lem:alpha} that for any $s\in\S$ and any $t\le k$,
		\begin{equation}\label{eqa:prop2 ***}
			\left( \sum_{u\in\{0,1\}^n} \alpha_u  \Pr\big(W_t=u\mid W\in\bs\big) \right)^2\, \le\, 1.5\Big(n-H\big(W_t\mid W\in\bs\big)\Big).
		\end{equation}

		In what follows, we try to derive a bound on $P\big(\bs\mid t\big)/P(\bs)$ in terms of $\alpha_u$. We then use \eqref{eqa:prop2 ***} and Lemma~\ref{lem:HI} to obtain the desired bound on $I(T;S)$. We now elaborate on $P\big(\bs\mid t\big)$, 
		\begin{equation}\label{eqa:ps|t equality}
			\begin{split}
				P\big(\bs\mid T=t\big) \,&=\, 2^{-n(k-1)} \sum_{w\in\bs} P\big(w_t\mid T=t\big)\\
				&=\, 2^{-n(k-1)} \sum_{u\in\{0,1\}^n} \Pr\big(W_t=u\mid T=t\big) N_t^{\bs}(u)\\
				&=\, 2^{-nk} \sum_{u\in \{0,1\}^n} \Big(2^n \Pr\big(W_t=u\mid T=t\big)\Big) N_t^{\bs}(u)\\
				&=\, 2^{-nk} \sum_{u\in \{0,1\}^n} \Big((\alpha_u+1)\gamma\Big) N_t^{\bs}(u),
			\end{split}
		\end{equation}
		where the last equality is due to \eqref{eqa:notinU then Nu=0} and the definition of $\alpha_u$ in \eqref{eqa:def alpha u}.
		On the other hand, since $P(\bs)=\sum_{w\in\bs}P(w)$, it follows from \eqref{eqa:lemI pw} that 
		\begin{equation}\label{eqa:bs bound beta 1}
			\begin{split}
				P(\bs)\, &\le\, \sum_{w\in\bs} (1+\delta)\,2^{-kn} \,=\, 2^{-kn}\, (1+\delta) \,|\bs|,      \\
				P(\bs)\, &\ge\, \sum_{w\in\bs} (1-\delta)\,2^{-kn} \,=\, 2^{-kn}\, (1-\delta)\,|\bs|.
			\end{split}
		\end{equation}
		Combining \eqref{eqa:ps|t equality} and \eqref{eqa:bs bound beta 1}, %and employing the definition of $q_t^{\bs}(u)$ from \eqref{eqa:def q}, 
		we obtain
		\begin{equation}\label{eqa:prop2 **}
			\begin{split}
				\frac{P\big(\bs\mid T=t\big)}{P(\bs)} 
				\,&\le\,\frac{1}{1-\delta} \sum_{u\in\{0,1\}^n} (1+\alpha_u)\frac{N_t^{\bs}(u)}{|s|}, \\% \,=\, \frac{2^n}{(1+\delta)\beta} \sum_{u\in\{0,1\}^n} (1+\alpha_u) q_t^{\bs}(u),\\
				\frac{P\big(\bs\mid T=t\big)}{P(\bs)} \,&\ge\,\frac{1}{1+\delta} \sum_{u\in\{0,1\}^n} (1+\alpha_u)\frac{N_t^{\bs}(u)}{|s|}.% \,=\, \frac{2^n}{(1+\delta)\beta} \sum_{u\in\{0,1\}^n} (1-\alpha_u) q_t^{\bs}(u).
			\end{split}
		\end{equation}
		%where the first inequality are due to \eqref{eqa:bs bound beta 1} and the first equality is from \eqref{eqa:ps|t equality}.
		It then follows from \eqref{eqa:prop2 **} and \eqref{eqa:prop2 *1} that 
		\begin{equation}\label{eqa:*a1}
			\begin{split}
				\frac{P\big(\bs\mid T=t\big)}{P(\bs)} \,&\ge\,\frac{1}{1+\delta} \sum_{u\in\{0,1\}^n} (1+\alpha_u)\frac{N_t^{\bs}(u)}{|s|}\\
				&\ge\,\frac{1-\delta}{(1+\delta)^2} \sum_{u\in\{0,1\}^n} \big(1+\alpha_u\big) \Pr\big(W_t=u\mid W\in \bs\big)\\
				&\ge\, \big(1-4\delta)\, \sum_{u\in\{0,1\}^n} \big(1+\alpha_u\big) \Pr\big(W_t=u\mid W\in \bs\big) \\
				&\ge\, \sum_{u\in\{0,1\}^n} \big(1+\alpha_u\big) \Pr\big(W_t=u\mid W\in \bs\big) \, -\,4\delta \sum_{u\in\{0,1\}^n} 2\Pr\big(W_t=u\mid W\in \bs\big)\\
				&=\, \sum_{u\in\{0,1\}^n} \big(1+\alpha_u\big) \Pr\big(W_t=u\mid W\in \bs\big) \, -\,8\delta\\
				&=\, \sum_{u\in\{0,1\}^n} \alpha_u \Pr\big(W_t=u\mid W\in \bs\big) \, +\,1\,-\,8\delta,
			\end{split}
		\end{equation}
		where the first inequality is from~\eqref{eqa:prop2 **}, the second inequality follows from \eqref{eqa:prop2 *1}, the fourth inequality is because $\alpha_u\le 2$ for all $u\in\{0,1\}^n$, and the third inequality is due to the assumption that $\delta\le 1/7$ (see \eqref{eqa:lower condition 3}) and the following inequality (which is easy to verify with a computer program)
		\begin{equation*} 
			\frac{1-x}{(1+x)^2} \ge 1-4x \quad \textrm{and}  \quad \frac{1+x}{(1-x)^2} \le 1+4x,\qquad \forall x \in[0,1/7].
		\end{equation*}
		Following a similar line of arguments and using \eqref{eqa:prop2 *2} instead of \eqref{eqa:prop2 *1}, we obtain
		\begin{equation} \label{eqa:*a2}
			\frac{P\big(\bs\mid T=t\big)}{P(\bs)} \,\le\, \sum_{u\in\{0,1\}^n} \alpha_u \Pr\big(W_t=u\mid W\in \bs\big) \, +\,1\,+\,8\delta.
		\end{equation}
		Combining \eqref{eqa:*a1} and \eqref{eqa:*a2}, we obtain
		\begin{equation}\label{eqa:*a3}
			\begin{split}
				\left(\frac{P\big(\bs\mid T=t\big)}{P(\bs)} -1  \right)^2 \, &\le\, \left( \Big|\sum_{u\in\{0,1\}^n} \alpha_u \Pr\big(W_t=u\mid W\in \bs\big)\Big| \, +\,8\delta   \right)^2\\
				&\le\, 2\left( \sum_{u\in\{0,1\}^n} \alpha_u \Pr\big(W_t=u\mid W\in \bs\big)    \right)^2\, +\,2(8\delta)^2\\
				&\le\,  3\Big(n-H\big(W_t\mid W\in\bs\big)\Big)  \, +\,128\delta^2,
			\end{split}
		\end{equation}
		where the first inequality is due to \eqref{eqa:*a1} and \eqref{eqa:*a2}, the second inequality is because $(a+b)^2\le 2a^2+2b^2$, for all $a,b\in\R$, and the last inequality follows from \eqref{eqa:prop2 ***}.
		
		On the other hand, %it follows from Lemma~\ref{lem:pw bound} that %for any $w\in\bs$,
		%\begin{equation}
		%	\frac{P(w)}{P(\bs)}\, =\, \frac{P(w)}{\sum_{w'\in\bs}P(w')} \,\le\,\frac{(1-\delta)\,2^{-kn}}\frac{(1+\delta)\,2^{-kn} \, |\bs|}
		%	\,=\, \frac{1-\delta}\frac{(1+\delta)\, |\bs|}. \red{[REMOVE THE EQUATION]}
		%\end{equation}
		%Then,
		\begin{equation}\label{eqa:prop2 *4}
			\begin{split}
				H\big(W\mid W\in\bs\big)\,&=\, \sum_{w\in\bs} P\big(w\mid w\in\bs\big)\, \log_2\frac1{P\big(w\mid w\in\bs\big)}\\
				&=\, \sum_{w\in\bs} \frac{P(w)}{P(\bs)}\, \log_2\frac{P(\bs)}{P(w)}\\
				&\ge\, \sum_{w\in\bs} \frac{P(w)}{P(\bs)}\, \log_2\frac{P(\bs)}{(1+\delta)\, 2^{-kn}}\\
				&=\,  \log_2\frac{P(\bs)}{(1+\delta)\, 2^{-kn}}\\
				&=\, kn\,+ \, \log_2 P(\bs) \, -\, \log_2 (1+\delta)\\
				&\ge\, kn\, -\, 1.5\delta \,+ \, \log_2 P(\bs) ,
			\end{split}
		\end{equation}
		where the first inequality is due to Lemma~\ref{lem:pw bound}, and the last inequality is because $\log_2 (1+x) \le 1.5 x$, for all $x> -1$.
		Moreover,
		\begin{equation}\label{eqa:ent subadd}
			H\big(W\mid W\in\bs\big) \, =\, H\big(W_1,\ldots,W_k\mid W\in\bs\big) \,\le\, \sum_{t=1}^{k} H\big(W_t\mid W\in\bs\big),
		\end{equation}
		where the inequality is from the sub-additive property of the entropy (see \cite{CoveT06}, page 41).
		Plugging \eqref{eqa:prop2 *4} into \eqref{eqa:ent subadd}, we obtain
		\begin{equation}\label{eqa:prop2 5*}
			\sum_{t=1}^{k} H\big(W_t\mid W\in\bs\big)\, \ge\, 	H\big(W\mid W\in\bs\big)\, \ge\, \log_2 P(\bs) \,+ \,  kn\, -\, 1.5\delta.
		\end{equation}
		Combining everything together, we finally have
		\begin{equation*}
			\begin{split}
				I(T;S)\,&\le\, \frac1{k \ln 2} \sum_{s\in\S} P(\bs) \sum_{t=1}^k \left( \frac{P(\bs\mid T=t)}{P(\bs)} \,-1 \right)^2 \,+\, \frac{9+6B + 15\sqrt{k}}{k^3}\\
				&\le\,\frac1{k \ln 2} \sum_{s\in\S} P(\bs) \sum_{t=1}^k \Big( 3n-3H\big(W_t\mid W_t\in\bs\big) +128\delta^2 \Big) \,+\, \frac{9+6B + 15\sqrt{k}}{k^3}\\
				&=\, \frac{3n}{k \ln 2}\,-\,\frac3{k \ln 2} \sum_{s\in\S} P(\bs) \sum_{t=1}^k H\big(W_t\mid W_t\in\bs\big)  \, +\,\frac{128\delta^2}{\ln 2} + \frac{9+6B + 15\sqrt{k}}{k^3}\\
				&\le\, \frac{3n}{k \ln 2}\,-\,\frac3{k \ln 2} \sum_{s\in\S} P(\bs)  \Big(\log_2 P(\bs) \,+\,nk-1.5\delta \Big)  \, +\,185\delta^2 + \frac{9+6B + 15\sqrt{k}}{k^3}\\
				&=\, \frac3{k \ln 2} \sum_{s\in\S} P(\bs) \log_2 \frac1{P(\bs)} \,+\, \frac{4.5\delta}{k \ln 2}  \, +\,185\delta^2 + \frac{9+6B + 15\sqrt{k}}{k^3}\\
				&\le\, \frac3{k \ln 2}\, H(S) \,+\, \frac{27}{k^3 \ln 2} \,+\, \frac{6.5\delta}{ k}  \, +\,185\delta^2 + \frac{9+6B + 15\sqrt{k}}{k^3}\\
				&\le\, \frac{3B}{k \ln 2}\,+\, \frac{40}{k^3 } \,+\, \frac{6.5\delta}{ k}  \, +\,185\delta^2 + \frac{9+6B + 15\sqrt{k}}{k^3}\\
				&< \, \frac{3B}{k \ln 2} \,+\, \frac1k\,\left[ \left(\frac{313}{\cn}\right)^2\,+\, \frac{94^2}{\cn \sqrt{k}} \,+\, \frac{192}k \,+\frac{15}{k^{1.5}} \,+\, \frac{49+6B}{k^2}  \right],
			\end{split}
		\end{equation*}
		where the first inequality follows from Lemma~\ref{lem:HI}~(b), the second inequality is due to \eqref{eqa:*a3}, the third inequality is from \eqref{eqa:prop2 5*}, the fourth inequality results from Lemma~\ref{lem:HI}~(a), the fifth inequality is because $S$ is a $B$-bit signal and as a result, $H(S)\le B$, and the last inequality is by substituting $\delta$ from \eqref{eqa:def del} and simple calculations. This implies \eqref{eqa:bound I} and completes the proof of Proposition~\ref{prop:I}.

		\medskip
		%%%%%%%%%%%%%%%%%%%%%%%%%%%%%%%%%%%%%%%%%%%%%%%%%%%%%%%%%%%%%%%%%%%%%%%%%%%%%%%%%%%%%%%
		
		\subsection{Proof of Lemma~\ref{lem:deterministic coding}} \label{app:proof lem deterministic coding}
		The proof relies on a known property of mutual information (see Theorem~2.7.4 of \cite{CoveT06} on page 33), according to which 
		\begin{equation}\label{eqa:I is conv}
			I(S;T) \textrm{ is a convex function with respect to } P\big(S\mid T\big).
		\end{equation}
		Let $Q$  be a randomized coding, under which a machine outputs signal $S$ given the coin-flipping outcome vector $W$ with probability $Q\big(S\mid W\big)$.
		For any $s\in\S$ and $t=1,\ldots,k$, let
		\begin{equation}\label{eqa:def PF}
			P_Q\big(s\mid t\big)\,\triangleq\, \sum_{w\in\W} P\big(w\mid t\big)\, Q\big(s\mid w\big)
		\end{equation}
		be the probability of signal $s$ given the biased coin index $t$.
		Let $\Q$ be the set of all deterministic mappings (or functions) from $\W$ to $\S$. Corresponding to any $g\in\Q$, we consider a deterministic coding $Q_g$ as follows
		\begin{equation}\label{eqa:def Fg}
			Q_g\big(s\mid w \big)\,=\, \begin{cases}
				1& \quad \textrm{if } g(w)=s,\\
				0& \quad \textrm{otherwise.}
			\end{cases}
		\end{equation}
		We also let 
		\begin{equation}\label{eqa:def Pg}
			P_g\big(s\mid t\big)\,\triangleq\, \sum_{w\in\W} P\big(w\mid t\big)\, Q_g\big(s\mid w\big)
		\end{equation}
		be the probability of signal $s$ given the biased coin index $t$, under the coding $Q_g$.
		We will show that for any stochastic coding $Q$, $P_Q(\cdot)$ is a convex combination of  $P_g(\cdot)$, for $g\in\Q$, in the sense that there exist non-negative coefficients $\alpha_g$, for $g\in\Q$, such that $\sum_{g\in\Q}\alpha_g=1$ and 
		\begin{equation}\label{eqa:conv comb1}
			P_Q\big(s\mid t\big) \, =\, \sum_{g\in\Q}\alpha_g\, P_g\big(s\mid t\big),\qquad \forall s\in\S,\quad t=1,\ldots,k.
		\end{equation}
		Once we establish \eqref{eqa:conv comb1}, it follows from \eqref{eqa:I is conv} that\footnote{Please note that $I(S;T)$ can be seen as a convex function of a vector $\alpha=(\alpha_1,\cdots,\alpha_{|\Q|})$ where $\sum_{g\in\Q}\alpha_g=1$. Moreover, the value of this function at the standard basis vector $e_i$, $1\leq i \leq |\Q|$, would be $I\big(S_g;T\big)$. Thus, the value of $I(S;T)$ is less than the linear combination of values of this function at basis vectors with weights given in $\alpha$.   }
		\begin{equation}\label{eqa:conv I last}
			I\big(S;T\big) \,\le\, \sum_{g\in\Q}\alpha_g I\big(S_g;T\big) \, \le\, \max_{g\in\Q} I\big(S_g;T\big),
		\end{equation}
		where $S$ is a random signal generated via coding $Q$, and for $g\in\Q$, $S_g$ is a random signal generated under coding $Q_g$. As a result, there exists a $g\in\Q$ such that the mutual information under deterministic coding $Q_g$ is no smaller than the mutual information under the randomized coding $Q$. This shows that the mutual information is maximized under a deterministic coding, which in turn implies the lemma.
		In the rest of the proof, we will establish \eqref{eqa:conv comb1}.
		
		Lets fix a randomized coding $Q$. We enumerate the set $\W$ and let $\W=\big\{w^{1},\ldots,w^{2^{kn}}\big\}$. 
		For any $g\in\Q$ let
		\begin{equation}\label{eqa:def alpha g}
			\alpha_g \,\triangleq\,  \prod_{w\in\W}Q\big( g(w) \mid w\big)\, = \,  \prod_{i=1}^{2^{kn}} Q\big( g(w^i) \mid w^i\big).
		\end{equation}
		Then,
		\begin{equation}
			\begin{split}
				\sum_{g\in\Q}\alpha_g \,&=\, \sum_{g\in\Q}\, \prod_{i=1}^{2^{kn}} Q\big( g(w^i) \mid w^i\big)\\
				&=\, \sum_{s_1\in\S}\cdots\sum_{s_{2^{kn}}\in\S}\, \prod_{i=1}^{2^{kn}} Q\big( s_i \mid w^i\big)\\
				&=\, \left(\sum_{s_1\in\S} Q\big( s_1 \mid w^1\big)\right) \times \cdots \times \left(\sum_{s_{2^{kn}}\in\S} Q\big( s_{2^{kn}} \mid w^{2^{kn}}\big)\right)\\
				&=\, 1\times\cdots\times 1\\
				&=\, 1,
			\end{split}
		\end{equation}
		where the second equality is because $\Q$ is the set of all deterministic functions from $\W$ to $\S$ and for any $s_1,\ldots,s_{2^{kn}}\in \S$, there exists a $g\in\Q$ such that $g(w^i)=s_i$ for $i=1,\ldots,2^{kn}$; and the last inequality is because for any $w\in\W$, $Q\big(\cdot\mid w\big)$ is a probability mass function over $\S$.
		
		On the other hand, for any $s\in\S$,
		\begin{equation}
			\begin{split}
				\sum_{g\in\Q}\alpha_g  Q_g\big( s \mid w^1\big) \,&=\, \sum_{\substack{g\in\Q \\ g(w^1)=s}} \alpha_g\\
				&=\, \sum_{\substack{g\in\Q \\ g(w^1)=s}} \prod_{i=1}^{2^{kn}} Q\big( g(w^i) \mid w^i\big)\\
				&=\, Q\big( s \mid w^1\big)\,\sum_{\substack{g\in\Q \\ g(w^1)=s}} \prod_{i=2}^{2^{kn}}  Q\big( g(w^i) \mid w^i\big)\\
				&=\, Q\big( s \mid w^1\big)\,\sum_{s_2\in\S}\cdots\sum_{s_{2^{kn}}\in\S} \, \prod_{i=2}^{2^{kn}}  Q\big( g(w^i) \mid w^i\big)\\
				&=\, Q\big( s \mid w^1\big)\,\left(\sum_{s_2\in\S} Q\big( s_2 \mid w^2\big)\right) \times \cdots \times \left(\sum_{s_{2^{kn}}\in\S} Q\big( s_{2^{kn}} \mid w^{2^{kn}}\big)\right)\\
				&=\,  Q\big( s \mid w^1\big)\times 1\times\cdots\times 1\\
				&=  Q\big( s \mid w^1\big),
			\end{split}
		\end{equation}
		where the first equality follows from the definition of $Q_g$ in \eqref{eqa:def Fg}, the fourth equality is  because for any $s_1,\ldots,s_{2^{kn}}\in \S$, there exists a $g\in\Q$ such that $g(w^i)=s_i$ for $i=1,\ldots,2^{kn}$, and the sixth equality is because for any $w\in\W$, $Q\big(\cdot\mid w^1\big)$ is a probability mass function over $\S$.
		In the same vein, for any $w\in\W$ and any $s\in\S$, we have 
		\begin{equation} \label{eqa:FFg}
			Q\big( s \mid w\big)\,=\,\sum_{g\in\Q}\alpha_g  Q_g\big( s \mid w\big) . 
		\end{equation}
		Therefore, for $t=1,\ldots, k$ and for any $s\in\S$,
		\begin{equation}\label{eqa:conv comb2}
			\begin{split}
				P_Q\big(s\mid t\big) \, &=\,  \sum_{w\in\W} P\big(w\mid t\big)\, Q\big(s\mid w\big)\\
				&=\,\sum_{w\in\W} P\big(w\mid t\big)\, \sum_{g\in\Q}\alpha_g\,  Q_g\big(s\mid w\big)\\
				&=\,\sum_{g\in\Q}\alpha_g\, \sum_{w\in\W} P\big(w\mid t\big)\,  Q_g\big(s\mid w\big)\\
				&=\,\sum_{g\in\Q}\alpha_g  P_g\big(s\mid t\big),
			\end{split}
		\end{equation}
		where the first equality is from the definition of $P_Q(\cdot)$ in \eqref{eqa:def PF},
		the second equality follows from \eqref{eqa:FFg},
		and the last equality is due to the definition of $P_g(\cdot)$ in \eqref{eqa:def Pg}. 
		This implies \eqref{eqa:conv comb1}. Lemma~\ref{lem:deterministic coding} then follows from the argument following \eqref{eqa:conv I last}.

		\medskip
		%%%%%%%%%%%%%%%%%%%%%%%%%%%%%%%%%%%%%%%%%%%%%%%%%%%%%%%%%%%%%%%%%%%%%%%%%%%%%%%%%%%%%%%%%%%
		
		\subsection{Proof of Lemma~\ref{lem:pw bound}} \label{app:proof lem pw bound}
		Fix a $t_0\le k$ and let $\psi$ be a random  outcome of the coin flipping matrix generated via distribution $P(W\mid T=t_0)$.
		For $t=1,\dots,k$ let $\delta_t$  denote the number of $1$s in the $t$th column of $\psi$.
		Therefore,
		\begin{equation}\label{eqa:E deltat0}
			\Exp\big[\delta_{t_0}\big] \, = \frac{n}2\,+\,\frac{\sqrt{n}}{2\cn \,\ln k},
		\end{equation}
		and for any $t\ne t_0$,
		\begin{equation}\label{eqa:E deltat}
			\Exp\big[\delta_{t}\big] \, = \frac{n}2.
		\end{equation}
		
		We now capitalizing on the Hoeffding's inequality (see Lemma~\ref{lemma:CI}~(a)) to obtain
		\begin{equation}
			\begin{split}
				\Pr\left(\big|\delta_{t_0} - \frac{n}2 \big| \ge 2.5\sqrt{n\ln k} \right)\,
				&\le \, \Pr\left(\big|\delta_{t_0} - \Big(\frac{n}2+\frac{\sqrt{n}}{2\cn \ln k}\Big) \big| \ge 2.5\sqrt{n\ln k} -\frac{\sqrt{n}}{2\cn \ln k}\right)\\
				&= \, \Pr\left(\big|\delta_{t_0} - \Exp[\delta_{t_0}] \big| \ge 2.5\sqrt{n\ln k} -\frac{\sqrt{n}}{2\cn \ln k}\right)\\
				&\le \, \Pr\left(\big|\delta_{t_0} - \Exp[\delta_{t_0}] \big| \ge 2\sqrt{n\ln k}\right)\\
				&\le\,2\exp\left( \frac{-8n \ln k}{n}   \right)\\
				&\le\,2\exp\big( -8 \ln k \big)\\
				&\le\, \frac{2}{k^4},
			\end{split}
		\end{equation}
		where the first equality is from \eqref{eqa:E deltat0} and the third inequality is due to the  Hoeffding's inequality.
		In the same vein, for any $t\ne t_0$,
		\begin{equation}
			\begin{split}
				\Pr\left(\big|\delta_{t} - \frac{n}2 \big| \ge 2.5\sqrt{n\ln k} \right)\,
				&= \, \Pr\left(\big|\delta_{t} - \Exp[\delta_{t}] \big| \ge 2.5\sqrt{n\ln k} \right)\\
				&\le\,2\exp\left( \frac{-12.5n \ln k}{n}   \right)\\
				&\le\, \frac{2}{k^4},
			\end{split}
		\end{equation}
		where the equality is due to \eqref{eqa:E deltat}
		Therefore, for $t=1,\ldots,k$,
		\begin{equation}\label{eqa:deltan/2k4}
			\Pr\left(\big|\delta_{t} - \frac{n}2 \big| \ge 2.5\sqrt{n\ln k} \right)\,\le\, \frac{2}{k^4}.
		\end{equation}

		It is easy to verify via a simple computer program that $e^x\le 1+4x/3$, for all  $x\in [0,0.5]$. It then follows from \eqref{eqa:lower condition 1} with $k=mB$ that
		\begin{equation}\label{eqa:lemI A*}
			\exp\left(\frac{15}{2\cn \sqrt{\ln k}}  \right)\,\le\, 1\,+\,\frac{10}{\cn\,\sqrt{\ln k}}.
		\end{equation}
		Let
		\begin{equation} \label{eqa:LemI def eps}
			\epsilon \,\triangleq\, \frac{1}{\cn \sqrt{n}\, \ln k}.
		\end{equation} 
		In the same vein, we have  $(1+x)/(1-x)\le e^{3x}$, for all $x\in [0,1/3]$.
		Therefore, in view of \eqref{eqa:lower condition 1}, $\epsilon\le 1/3$, and hence,
		\begin{equation}\label{eqa:lemI *3}
			\frac{1+\epsilon}{1-\epsilon}\,\le\, e^{3\epsilon}.
		\end{equation}
		Moreover, for any $x\in[0,0.5]$, we have $1-x\ge e^{-2x}$. Consequently,
		\begin{equation}\label{eqa:LemI *5}
			\left(1-\epsilon^2\right) ^{n/2} \,\ge\,\big( \exp(-2\epsilon^2)\big)^{n/2} \,=\,\exp\big(-n\epsilon^2\big)\,=\, \exp\left(\frac{-1}{\cn^2 \,\ln^2 k}\right).
		\end{equation}

		Once again, we emphasize that $\psi$ is sampled from a distribution in which the $t_0$th coin is biased. Then, for $t=1,\ldots,k$,
		\begin{equation}\label{eqa:lemI *2}
			\begin{split}
				\Pr\big(W_t=\psi_t \mid T=t\big)\, &=\,  \left(\frac{1}2 \,+ \frac{\epsilon}2\right)^{\delta_t} \,\left(\frac{1}2 \,- \frac{\epsilon}2\right)^{n-\delta_t}\\
				&=\,  2^{-n}\left(1+ \epsilon\right)^{\frac{n}2 + (\delta_t-\frac{n}2)} \,\left(1- \epsilon\right)^{\frac{n}2 - (\delta_t-\frac{n}2)}\\
				&=\,  2^{-n}\left(1- \epsilon^2\right)^{\frac{n}2 } \,\left(\frac{1+ \epsilon}{1-\epsilon}\right)^{\delta_t-\frac{n}2},
			\end{split}
		\end{equation}
		where the first equality is due to \eqref{eqa:prob of biased coin} and the definition of $\epsilon$ in \eqref{eqa:LemI def eps}.
		Assuming $|\delta_t-n/2|\le 2.5 \sqrt{n\ln k}$, \eqref{eqa:lemI *2} simplifies to 
		\begin{equation}\label{eqa:lemI *4}
			\begin{split}
				P\big(W_t=\psi_t \mid T=t\big)\, &\le\,  2^{-n}\left(\frac{1+ \epsilon}{1-\epsilon}\right)^{\delta_t-\frac{n}2}\\
				&\le\, 2^{-n}\left(e^{3\epsilon}\right)^{|\delta_t-\frac{n}2|}\\
				&\le\, 2^{-n}\exp\left(7.5 \epsilon \sqrt{n \ln k}\right)\\
				&=\, 2^{-n}\exp\left(\frac{15}{2\cn \sqrt{\ln k}}\right)\\
				&\le\, 2^{-n}\,\left(1\,+\,\frac{10}{\cn \sqrt{\ln k}}\right),
			\end{split}
		\end{equation}
		where the second inequality follows from \eqref{eqa:lemI *3}, the third inequality is due to the assumption $|\delta_t-n/2|\le 2.5 \sqrt{n\ln k}$, the equality is by the definition of $\epsilon$ in \eqref{eqa:LemI def eps}, and the last inequality is from \eqref{eqa:lemI A*}.
		
		In the same vein, assuming $|\delta_t-n/2|\le 2.5 \sqrt{n\ln k}$, \eqref{eqa:lemI *2} can be simplified as
		\begin{equation} \label{eqa:lemI *6}
			\begin{split}
				P\big(W_t=\psi_t \mid T=t\big)\, &=\, 2^{-n}\left(1- \epsilon^2\right)^{\frac{n}2 } \,\left(\frac{1+ \epsilon}{1-\epsilon}\right)^{\delta_t-\frac{n}2}	\\
				&\ge\,  2^{-n} \,\exp\left(\frac{-1}{\cn^2 \,\ln^2 k}\right) \, \left(\frac{1+ \epsilon}{1-\epsilon}\right)^{-|\delta_t-\frac{n}2|}\\
				&\ge\,  2^{-n} \,\exp\left(\frac{-1}{\cn^2 \,\ln^2 k}\right) \, \exp\left(3\epsilon\right)^{-|\delta_t-\frac{n}2|}\\
				&\ge\,  2^{-n} \,\exp\left(\frac{-1}{\cn^2 \,\ln^2 k} \, -\, 7.5\epsilon\sqrt{n \ln k}\right) \\
				&=\,  2^{-n} \,\exp\left(\frac{-1}{\cn^2 \,\ln^2 k} \, -\, \frac{15}{2\cn \sqrt{\ln k}}\right) \\
				&\ge\,  2^{-n} \,\exp\left(\frac{-8}{\cn \sqrt{\ln k}}\right) \\
				&\ge\,  2^{-n} \,\left(1 -\,\frac{8}{\cn \sqrt{\ln k}}\right),
			\end{split}
		\end{equation}
		where the first inequality follows from \eqref{eqa:LemI *5}, the second inequality is due to \eqref{eqa:lemI *3}, the third inequality is from the assumption that $|\delta_t-n/2|\le 2.5 \sqrt{n\ln k}$, the second equality is from the definition of $\epsilon$ in \eqref{eqa:LemI def eps}, the fourth inequality is because $\cn \ln k\ge 2$ (see \eqref{eqa:lower condition 1} with identification $k=mB$), and the last inequality is because $e^{-x}\ge 1-x$, for all $x\in\R$.
		
		Combining \eqref{eqa:deltan/2k4}, \eqref{eqa:lemI *4}, and \eqref{eqa:lemI *6}, it follows that % when $\omega$ is drawn from a distribution in which the biased coin is $t_0$, then 
		for $t=1,\ldots,k$,  with probability at least $1-2k^{-4}$ we have
		\begin{equation} \label{eqa:lemI *7}
			\Pr\big(W_t=\psi_t \mid T=t  \big)\, \in\, 2^{-n}\,\times\, \left( 1 -\,\frac{8}{\cn \sqrt{\ln k}},\, 1 +\,\frac{10}{\cn \sqrt{\ln k}}   \right).
		\end{equation}
		
		Let $\Wb_1$ be a subset of $\W$ that contains all $w\in \W$ for which $|\delta_t-n/2|\le 2.5 \sqrt{n\ln k}$, for $t=1,\ldots, k$. Then, from \eqref{eqa:lemI *4} and \eqref{eqa:lower condition 1},  for any $w\in\Wb_1$, we obtain:
		\begin{equation}\label{eqa:Pt 2}
			\Pr\big(W_t=w_t \mid T=t  \big)\, \le\, 2^{-n}  \left(1+\frac{10}{15}\right)\,=\, \frac{5\times 2^{-n}}{3}.
		\end{equation}
		Moreover, it  follows from \eqref{eqa:deltan/2k4} and the union bound that 
		\begin{equation}\label{eqa:lemI *8}
			P\big(\Wb_1\big) \, \ge\, 1-2k^{-3}.
		\end{equation}
		
		We now proceed to prove the second part of the lemma, i.e. \eqref{eqa:lemI pw}.
		Again, fix a $t_0\le k$ and let $\psi$ be a random matrix of coin-flip outcomes in which the biased coin has index $T=t_0$.
		In this case, the columns $\psi_1,\ldots,\psi_k$ of $\psi$ are independent random vectors. 
		For $t=1,\ldots,k$, let 
		\begin{equation}\label{eqa:yt}
			y_t\, =\, f(\psi_t)\, \triangleq \, \min\Bigg(\max\left( 2^n \Pr\big(W_t=\psi_t \mid T=t  \big),\,\,  1 -\,\frac{8}{\cn \sqrt{\ln k}}\right),\,\,  1 +\,\frac{10}{\cn \sqrt{\ln k}} \Bigg).
		\end{equation}
		Since each $y_t$ is only a function of $\psi_t$, it follows that  $y_1,\ldots, y_k$ are independent random variables. Moreover, every $y_t$ lies in an interval of length ${18}/\cn \sqrt{\ln k}$. 
		On the other hand, it follows from \eqref{eqa:lemI *7} that for  $t=1,\ldots,k$, with probability at least $1-2k^{-4}$, we have $y_t=2^n \Pr\big(W_2=\psi_t \mid T=t  \big)$. The union bound then implies that with probability at least $1-2k^{-3}$,
		\begin{equation}\label{eqa:lemI *5kahi}
			y_t=2^n \Pr\big(W_t=\psi_t \mid T=t  \big),\qquad \textrm{ for } t=1,\ldots,k.
		\end{equation}
		Therefore, for the random matrix $\psi$ sampled from a distribution with biased coin index $T=t_0$, we have
		\begin{equation}
			\begin{split}
				P(\psi)\,&=\,\Pr(W=\psi)\, \\&=\, \frac1{k} \sum_{t=1}^{k} \Pr\big(W=\psi\mid T=t\big)\\
				&=\, \frac1{k} \sum_{t=1}^{k} 2^{-(k-1)n}\, \Pr\big(W_t=\psi_t\mid T=t\big)\\
				&=\, \frac{2^{-kn}}{k} \sum_{t=1}^{k} 2^{n}\, \Pr\big(W_t=\psi_t\mid T=t\big)\\
			\end{split}
		\end{equation}
		It then follows from \eqref{eqa:lemI *5kahi} that with probability at least $1-2k^{-3}$, 
		\begin{equation}\label{eqa:lemI **8.5}
			P(\psi) \,=\, \frac{2^{-kn}}{k} \sum_{t=1}^{k} y_t.
		\end{equation}
		Let 
		\begin{equation}\label{eqa:lemI def beta}
			\beta\triangleq \E\left[\frac{2^{-kn}}{k} \sum_{t=1}^{k} y_t\right].
		\end{equation}
		% Note that in view of \eqref{eqa:lower condition 1}, for any $t\le k$, we have $y_t\ge 1-8/\cn \sqrt{\ln k}\ge 1-8/15> 0.46$. Therefore, $\beta> 0.46\times2^{-kn}$.
		\begin{claim}\label{claim:beta}
			$|\beta-2^{-kn}|\le 2^{-kn}/k$.
		\end{claim}
		\begin{proof}
			%[Proof of Claim]
			Temporarily,  fix  a $t\ne t_0$ and let
			\begin{align}
				\U^+ \,&\triangleq\,  \left\{ u\in \{0,1\}^n \quad\Big|\quad 2^n\Pr\big(\psi_t=u\mid T=t \big) > 1+\frac{10}{\cn\sqrt{\ln k}}\right\}, \label{eqa:c def U+}\\
				\U^- \,&\triangleq\,  \left\{ u\in \{0,1\}^n \quad\Big|\quad 2^n\Pr\big(\psi_t=u\mid T=t \big) < 1-\frac{8}{\cn\sqrt{\ln k}}\right\}. \label{eqa:c def U-}
			\end{align}
			Then, it follows from \eqref{eqa:lemI *7} that
			\begin{equation} \label{eqa:c *}
				\sum_{u\in\U^+} \Pr\big(\psi_t=u\mid T=t \big) \,\le\, 2k^{-4}.
			\end{equation}
			On the other hand, \eqref{eqa:lemI *6} implies that for any $u\in \U^-$
			\begin{equation}\label{eqa:c **1}
				|\delta(u)-n/2|\ge 2.5 \sqrt{n\ln k}, 
			\end{equation}
			where  $\delta(u)$ is the number of $1$s in the binary vector $u$.
			Let $z_1,\ldots,z_n$ be i.i.d. binary outcomes of a fair coin flip, and let $Z=z_1+\cdots+z_n$. Then,
			\begin{equation}\label{eqa:c **}
				\begin{split}
					\sum_{u\in\U^-}\frac1{2^n} \, &\le\, \sum_{\substack{u\in\{0,1\}^n \\ |\delta(u)-n/2|\ge 2.5 \sqrt{n\ln k}}} \frac1{2^n}\\
					&=\, \Pr\left( \big|Z-\frac{n}2\big|\,\ge\, 2.5 \sqrt{n\ln k}  \right)\\
					&\le\, 2\exp\left( \frac{2\times \left(2.5 \sqrt{n\ln k}\right)^2}{n}  \right)\\
					&\le\, 2\exp\big( -4 \ln k  \big)\\
					&=\, 2k^{-4},
				\end{split}
			\end{equation}
			where the first inequality follows from \eqref{eqa:c **1}, the first equality is because $Z$ has uniform distribution over $\{0,1\}^n$, and the second inequality is due to the Hoeffding's inequality.
			
			We now expand $\Exp[y_t]$ as follows. From \eqref{eqa:yt}, we have
			\begin{equation}\label{eqa:c ***}
				\begin{split}
					\Exp[y_t]\, &=\, \Exp\big[f(\psi_t)\big]\\
					&=\, \sum_{u\in\{0,1\}^n}\Pr\big(\psi_t=u\big) f(u)    \\
					&=\, \frac{1}{2^n}\sum_{u\in\{0,1\}^n} f(u)    \\
					&=\, \frac{1}{2^n}\sum_{u\in\{0,1\}^n} \min\Bigg(\max\left( 2^n \Pr\big(\psi_t =u\mid T=t  \big),\,\,  1 -\,\frac{8}{\cn \sqrt{\ln k}}\right),\,\,  1 +\,\frac{10}{\cn \sqrt{\ln k}} \Bigg)  \\
					&=\,  \frac{1}{2^n}\,\Bigg( \sum_{u\in\{0,1\}^n} 2^n \Pr\big(\psi_t =u\mid T=t  \big)\\
					\,&\qquad+\, \sum_{u\in\U^+} \left[ \left(1+\frac{10}{\cn \sqrt{\ln k}}\right) \,-\, 2^n \Pr\big(\psi_t =u\mid T=t  \big)\right]\\
					\,&\qquad+\, \sum_{u\in\U^-} \left[ \left(1- \frac{8}{\cn \sqrt{\ln k}}\right) \,-\, 2^n \Pr\big(\psi_t =u\mid T=t  \big)\right]	 \Bigg)\\
					&=\,  1 \,-\, \sum_{u\in\U^+} \left[ \Pr\big(\psi_t =u\mid T=t  \big)\,-\,2^{-n}\left(1+\frac{10}{\cn \sqrt{\ln k}}\right) \right]\\
					\,&\qquad+\, \frac{1}{2^n} \sum_{u\in\U^-} \left[ \left(1- \frac{8}{\cn \sqrt{\ln k}}\right) \,-\, 2^n \Pr\big(\psi_t =u\mid T=t  \big)\right]	 \Bigg),
				\end{split}
			\end{equation}
			where the third equality is because $t\ne t_0$ and as a result $\Pr\big(\psi_t=u\big)=2^{-n}$ for all $u\in\{0,1\}^n$, and the fifth equality follows from the definitions of $\U^+$ and $\U^-$ in \eqref{eqa:c def U+} and \eqref{eqa:c def U-}, respectively.
			From \eqref{eqa:c ***}, we have
			\begin{equation}\label{eqa:c *3}
				\Exp[y_t] \,\ge\, 1\,-\,\sum_{u\in\U^+}  \Pr\big(\psi_t =u\mid T=t  \big) \,\ge\, 1-2k^{-4},
			\end{equation}
			where the second inequality is due to \eqref{eqa:c *}.
			Moreover, it follows from \eqref{eqa:c ***} that
			\begin{equation}\label{eqa:c *4}
				\Exp[y_t] \,\le\, 1\,+\, \frac{1}{2^n} \sum_{u\in\U^-}  \left(1- \frac{8}{\cn \sqrt{\ln k}}\right)
				\,\le \, 1\,+\, \sum_{u\in\U^-} \frac{1}{2^n} 
				\,\le\, 1+2k^{-4},
			\end{equation}
			where the first inequality is due to \eqref{eqa:c ***}, and the last inequality is form \eqref{eqa:c **}.
			Combining  \eqref{eqa:c *3} and \eqref{eqa:c *4}, it follows that for any $t\ne t_0$,
			\begin{equation}\label{eqa:c *5}
				\big|\Exp[y_t] -1   \big|\, \le\, 2k^{-4}.
			\end{equation}
			On the other hand, \eqref{eqa:lower condition 1} implies that $\cn \sqrt{\ln k}\ge 15$. Therefore, from the definition of $y_t$, we have $y_{t_0}\in\big(1-8/15,\,1+10/15\big)$. Therefore,
			\begin{equation}\label{eqa:c *6}
				\big|\Exp[y_{t_0}] -1   \big|\, \le\, \frac23.
			\end{equation}
			Combining \eqref{eqa:c *5} and \eqref{eqa:c *6}, we obtain
			\begin{equation}
				\begin{split}
					\big|\beta - 2^{-kn}    \big| \, &=\, 2^{-kn} \, \Big|  \frac1k \sum_{t=1}^{k}\Exp[y_{t}]\,-\,1  \Big|\\
					&\le\, \frac{2^{-kn}}k \,\sum_{t=1}^{k}\big|  \Exp[y_{t}]-1  \big|\\
					&\le\, \frac{2^{-kn}}k \, \left[(k-1)\times 2k^{-4}\,+\, \frac23    \right]\\
					&\le\, \frac{2^{-kn}}k,
				\end{split}
			\end{equation}
			where the first equality is from the definition of $\beta$ in \eqref{eqa:lemI def beta}, the third inequality follows from \eqref{eqa:c *5} and \eqref{eqa:c *6}, and the last inequality is due to the assumption $k^{-3}\le1/6$ (see \eqref{eqa:lower condition 2} with identification $k=mB$).
			This completes the proof of Claim~\ref{claim:beta}.
		\end{proof}
		
		We proceed with the proof of the lemma.
		Since $y_1,\ldots,y_k$ are independent random variables over an interval of length ${18}/\cn \sqrt{\ln k}$, employing the Hoeffding's inequality we have %(see Lemma~\ref{lemma:CI}~(a)) combined with \eqref{eqa:lemI **8.5} imply that with probability at least $1-2k^{-3}$
		\begin{equation} \label{eqa:lemI *9}
			\begin{split}
				\Pr\bigg( \big| P(\psi) - 2^{-kn}  \big| \ge\, &\frac{23\,\times\,2^{-nk}}{\cn \sqrt{k}} +\frac{2^{-kn}}k  \bigg)\, 
				\le\,\Pr\left( \big| P(\psi) -\beta  \big| \ge\, \frac{23\,\times\,2^{-nk}}{\cn \sqrt{k}}   \right)\\
				&\le\, \Pr\left( \big| P(\psi) -\beta  \big| \ge\, \frac{23\,\times\,2^{-nk}}{\cn \sqrt{k}} \,\,\bigg|\,\, P(\psi) = \frac{2^{-kn}}{k} \sum_{t=1}^{k} y_t \right)\\
				&\qquad\, + \Pr\left( P(\psi) \ne \frac{2^{-kn}}{k} \sum_{t=1}^{k} y_t \right)\\
				%&=\, \Pr\left( \Big| \frac1{k} \sum_{t=1}^{k} y_t -2^{kn}\beta  \Big| \ge\, \frac{23}{\cn \sqrt{k}}   \right)\\
				%&\qquad\, + \Pr\left( P(\omega) \ne \frac{2^{-kn}}{k} \sum_{t=1}^{k} y_t \right)\\
				&\le\, \Pr\left( \Big| \frac1{k} \sum_{t=1}^{k} y_t -2^{kn}\beta  \Big| \ge\, \frac{23}{\cn \sqrt{k}}   \right) \,+\, 2k^{-3}\\
				&\le\, 2 \exp\left( \frac{-2k\, \big(23/\cn \sqrt{k}\big)^2}{\big(18/\cn \sqrt{\ln k}\big)^2}   \right) \,+\, 2k^{-3}\\
				&=\, 2 \exp\left( \frac{-2\times 23^2}{18^2}\, \ln k   \right) \,+\, 2k^{-3}\\
				&\le\, 2 \exp\big( -3 \ln k   \big) \,+\, 2k^{-3}\\
				&\le\, 4k^{-3},
			\end{split}
		\end{equation}
		where the first inequality follows from Claim~\ref{claim:beta}, the third inequality is due to  \eqref{eqa:lemI **8.5}  and the fourth inequality follows from the Hoeffding's inequality (see Lemma~\ref{lemma:CI}~(a))     and the definition of $\beta$ in \eqref{eqa:lemI def beta}.
		
		Since $t_0$ was chosen arbitrarily, \eqref{eqa:lemI *9} holds when the biased coin has any index in $1,\ldots,k$, and as a result it also holds when the biased coin is chosen uniformly at random from $1,\ldots,k$.
		Finally, we	 define a subset $\Wb_2\subset \W$ as 
		\begin{equation} \label{eqa:lemI *10}
			\Wb_2\,\triangleq \, \left\{ w\in \W\,: \quad \big| P(\psi) -2^{-kn}  \big| \le\, \frac{23\,\times\,2^{-nk}}{\cn \sqrt{k}} + \frac{2^{-kn}}k  \right\},
		\end{equation}
		and let $\Wb = \Wb_1\bigcap\Wb_2$. Employing a union bound on \eqref{eqa:lemI *8} and \eqref{eqa:lemI *9}, it follows that $P\big(\Wb\big)\ge 1-6k^{-3}$. 
		Moreover, Equations \eqref{eqa:lemI pwt} and \eqref{eqa:lemI pw} in the lemma statement follow from \eqref{eqa:Pt 2} and \eqref{eqa:lemI *10}, respectively.
		This completes the proof of Lemma~\ref{lem:pw bound}.

		\medskip
		%%%%%%%%%%%%%%%%%%%%%%%%%%%%%%%%%%%%%%%%%%%%%%%%%%%%%%%%%%%%%%%%%%%%%%%%%%%%%%%%%%%%%%%%%%
		
		\subsection{Proof of Lemma~\ref{lem:HI}} \label{app:proof lem HI}
		In this appendix, we present the proof of Lemma~\ref{lem:HI}.  
		For Part~(a), let $f(x)=x\log_2(1/x)$. Then $f'(x)=\log_2(1/x) - \log_2e$, where $e\simeq 2.718$ is the basis of the natural logarithm. As a result, $f'(x)\ge -1.5$, for all $x\in (0,1]$. Consequently, for any $s\in \S$,
		\begin{equation}\label{eqa:small firstodrder bound}
			f\big(P(s)\big) \,\ge\, f\big(P(\bs)\big) - 1.5 \big(P(s)-P(\bs)\big)
		\end{equation}
		Then,
		\begin{equation*}
			\begin{split}
				H(S) \,&= \, \sum_{s\in\S}  P(s)\log_2\frac1{P(s)} \\
				&= \, \sum_{s\in\S}  f\big(P(s)\big)\\
				&\ge \, \sum_{s\in\S} \Big( f\big(P(\bs)\big) - 1.5 \big(P(s)-P(\bs)\big) \Big)\\
				&= \, \left(\sum_{s\in\S} f\big(P(\bs)\big)\right) \,-\,  1.5 \left(\sum_{s\in\S} P(s)\,-\, \sum_{s\in\S}P(\bs)\right)\\
				&= \, \left(\sum_{s\in\S} P(\bs)\log_2\frac1{P(\bs)}\right) \,-\,  1.5 \big(1-P(\Wb)\big)\\
				\,&\ge\,\left(\sum_{s\in\S}  P(\bs)\log_2\frac1{P(\bs)}\right)\,-\, 9k^{-3},
			\end{split}
		\end{equation*}
		where the first equality is from the definition of entropy (see \cite{CoveT06}, page 14), and the inequalities are due to \eqref{eqa:small firstodrder bound} and Lemma~\ref{lem:pw bound}, respectively.
		This completes the proof of Part~(a) of the lemma.

		We now proceed to the proof of Part~(b). Following similar steps as in the proof of Part~(a), it can be shown than for $t=1,\ldots,k$, 
		\begin{equation}\label{eqa:lemHI **}
			\begin{split}
				H\big(S\mid T=t\big) \, &\triangleq\, \sum_{s\in\S} P\big(s\mid T=t\big)\log_2 \frac{1}{P\big(s\mid T=t\big)} \\
				&\ge\, \left(\sum_{s\in\S} P\big(\bs\mid T=t\big)\log_2 \frac{1}{P\big(\bs\mid T=t\big)}\right) \,- \, \frac9{k^3}.
			\end{split}
		\end{equation}
		Let $\Wb^c$ be the complement of the set $\Wb$, and for any $s\in\S$, let $\tilde{s}\triangleq s\bigcap \Wb^c$. Then, from Lemma~\ref{lem:pw bound}, we have
		\begin{equation} \label{eqa:stilde pwc}
			\sum_{s\in\S } P(\tilde{s}) \, =\, P\big(\Wb^c\big) \,\le\, 6k^{-3}.
		\end{equation}
		It is easy to verify that $x\log_2(1/x)\le 3.2x^{5/6}$, for all $x\ge0$. Then, 
		\begin{equation}\label{eqa:3.2k2.5}
			6k^{-3}\log_2 \frac{k^{3}}6 \, \le\, 3.2\,\big(6k^{-3}\big)^{5/6} \,\le\,  15 k^{-2.5}
		\end{equation}
		
		Let $|\S|$ be the number of elements in $\S$. Since $\S$ comprises the set of all $B$-bit signals, we have $|\S|=2^B$.
		It follows from the Jensen's inequality (see \cite{CoveT06}, page 25) that for fixed $\sum_{s\in\S } P(\tilde{s})$, the value of $\sum_{s\in\S } P(\tilde{s})\log_2\big(1/P(\tilde{s})\big)$ is maximized when all $P(\tilde{s})$, for $s\in\S$, have equal probability. 
		Therefore, 
		\begin{equation}\label{eqa:bound sum stilde}
			\begin{split}
				\sum_{s\in\S } P(\tilde{s})\log_2\frac1{P(\tilde{s})} \, &\le \, \sum_{s\in\S} \frac{\sum_{s\in\S } P(\tilde{s})}{|S|}\, \log_2 \frac{|S|}{\sum_{s\in\S } P(\tilde{s})}\\
				&= \, P\big(\Wb^c\big)\, \log_2 \frac{|S|}{P\big(\Wb^c\big)}\\
				%&= \, P\big(\Wb^c\big) \log_2 |S| \, +\, P\big(\Wb^c\big)\log_2\frac1{P\big(\Wb^c\big)}\\
				&= \, P\big(\Wb^c\big) B \, +\, P\big(\Wb^c\big)\log_2\frac1{P\big(\Wb^c\big)}\\
				&\le \, \frac{6B}{k^3} \, +\, P\big(\Wb^c\big)\log_2\frac1{P\big(\Wb^c\big)}\\
				&\le \, \frac{6B}{k^3} \, +\, 6k^{-3}\log_2 \frac{k^{3}}6\\
				&\le \, \frac{6B}{k^3} \, +\, \frac{15}{k^{2.5}}\\
				&= \, \frac{6B + 15\sqrt{k}}{k^3},
			\end{split}
		\end{equation}
		where the first equality follows from \eqref{eqa:stilde pwc}, the second equality is because $|\S|=2^B$,  the second inequality is due to \eqref{eqa:stilde pwc},  the third inequality is again because of \eqref{eqa:stilde pwc} and the fact that $x\log_2(1/x)$ is an increasing function for $x\in [0,1/e]$, and the last inequality follows from \eqref{eqa:3.2k2.5}.
		Consequently,
		\begin{equation} \label{eqa:lemHI *A}
			\begin{split}
				H(S)\, &=\, \sum_{s\in\S } P(s)\log_2\frac1{P(s)}\\
				&=\, \sum_{s\in\S } \big(P(\bs)+P(\tilde{s})\big)\,\log_2\frac1{P(\bs)+P(\tilde{s})}\\
				&\le\, \sum_{s\in\S } P(\bs)\log_2\frac1{P(\bs)} \, +\, \sum_{s\in\S } P(\tilde{s})\log_2\frac1{P(\tilde{s})}\\
				&\le\, \sum_{s\in\S } P(\bs)\log_2\frac1{P(\bs)} \, +\, \frac{6B + 15\sqrt{k}}{k^3},
			\end{split}
		\end{equation}
		where the last inequality is due to \eqref{eqa:bound sum stilde}.
		
		On the other hand, for any $x,y>0$, we have
		\begin{equation}\label{eqa:lemHI***}
			\begin{split}
				y\log_2y - x\log_2 x\,&= \, (y-x)\log_2x \,+\, y\big(\log_2 y - \log_2x\big)\\
				&=\, (y-x)\log_2x \,+\, y\log_2 \frac{y}{x}\\
				&\le\, (y-x)\log_2x \,+\, \frac{y}{\ln2}\,\left(\frac{y}{x}-1\right)\\
				&=\, (y-x)\log_2x \,+\, \frac{1}{\ln2}\left[\frac{y^2-yx}{x}\,+\,\frac{x^2-yx}{x}\,-\,\frac{x^2-yx}{x} \right]\\
				&=\, (y-x)\log_2x \,+\, \frac{1}{\ln2}\left[\frac{x^2+y^2-2yx}{x}\,-\,(x-y) \right]\\
				&=\, (y-x)\log_2(xe) \,+\, \frac{(x-y)^2}{x \ln2},
			\end{split}
		\end{equation}
		where the inequality is because $\log_2\alpha\le  (\alpha-1)/\ln2$, for all $\alpha>0$.
		Combining \eqref{eqa:lemHI **}, \eqref{eqa:lemHI *A}, and \eqref{eqa:lemHI***}, we obtain
		\begin{equation*}
			\begin{split}
				I(T;S)\,&=\, H(S)\,-\, \sum_{t=1}^k P\big(T=t\big) H\big(S\mid T=t\big)\\
				&=\, H(S)\,-\, \frac1k \sum_{t=1}^k H\big(S\mid T=t\big)\\
				&\le\, H(S)\,-\, \frac1k \sum_{t=1}^k  \left(\sum_{s\in\S} P\big(\bs\mid t\big)\log_2 \frac{1}{P\big(\bs\mid t\big)} \,- \, \frac9{k^3}\right)\\
				&\le\, \sum_{s\in\S } P(\bs)\log_2\frac1{P(\bs)} \, +\, \frac{6B + 15\sqrt{k}}{k^3}\\
				&\qquad-\, \frac1k \sum_{t=1}^k  \left(\sum_{s\in\S} P\big(\bs\mid t\big)\log_2 \frac{1}{P\big(\bs\mid t\big)} \,- \, \frac9{k^3}\right)\\
				&=\, \frac1k \sum_{t=1}^k  \sum_{s\in\S}\Big(  P\big(\bs\mid t\big)\log_2 P\big(\bs\mid t\big) \,-\, P(\bs)\log_2P(\bs)\Big) \, +\, \frac{9+6B + 15\sqrt{k}}{k^3}  \\
				&\le\, \frac1k \sum_{t=1}^k  \sum_{s\in\S}\left[  \Big(P\big(\bs\mid t\big)-P(\bs)\Big) \log_2\big(P(\bs)e\big) \,+\, \frac{\big(P(\bs)-P(\bs\mid t) \big)^2}{P(\bs)\ln2}\right] \\
				&\qquad +\, \frac{9+6B + 15\sqrt{k}}{k^3}  \\
				&=\, \sum_{s\in\S}  \log_2\big(P(\bs)e\big)   \left[ \left(\frac1k \sum_{t=1}^k P\big(\bs\mid t\big)\right)\,-\,P(\bs)\right]\\
				&\qquad+\,\frac1{k\ln2}   \sum_{s\in\S}P(\bs)\sum_{t=1}^k  \left(\frac{P(\bs)-P(\bs\mid t) }{P(\bs)}\right)^2 \, +\, \frac{9+6B + 15\sqrt{k}}{k^3}  \\
				&=\, \sum_{s\in\S}  \log_2\big(P(\bs)e\big)  \times 0\\
				&\qquad+\,\frac1{k\ln2}   \sum_{s\in\S}P(\bs)\sum_{t=1}^k  \left(\frac{P(\bs\mid t) }{P(\bs)}-1\right)^2 \, +\, \frac{9+6B + 15\sqrt{k}}{k^3}  \\
				&=\frac1{k\ln2}   \sum_{s\in\S}P(\bs)\sum_{t=1}^k  \left(\frac{P(\bs\mid T=t) }{P(\bs)}-1\right)^2 \, +\, \frac{9+6B + 15\sqrt{k}}{k^3},
			\end{split}
		\end{equation*}
		where the first equality is from the definition of mutual information (see \cite{CoveT06}, page 20), the first inequality is due to \eqref{eqa:lemHI **}, the second inequality is from \eqref{eqa:lemHI *A}, and the third inequality follows from \eqref{eqa:lemHI***}. 
		This completes the proof of Lemma~\ref{lem:HI}.

		\medskip
		%%%%%%%%%%%%%%%%%%%%%%%%%%%%%%%%%%%%%%%%%%%%%%%%%%%%%%%%%%%%%%%%%%%%%%%%%%%%%%%%
		
		\subsection{Proof of Lemma~\ref{lem:alpha}} \label{app:proof lem alpha}
		Let $\U^+$ be a subset of $\{0,1\}^n$ that contains the $2^{n-1}$ elements $u\in \{0,1\}^n$ with largest values of $P(u)$. 
		Also let  $\U^-$ be a subset of $\{0,1\}^n$ that contains the $2^{n-1}$ elements $u\in \{0,1\}^n$ with smallest values of $P(u)$. 
		Then $\U^+$ and $\U^-$ are disjoint sets with $\U^+\bigcup\U^-= \{0,1\}^n$. Moreover, for any $u\in \U^+$ and any $v\in \U^-$, we have $P(u)\ge P(v)$.
		Let $\theta\triangleq P(\U^+)=\sum_{u\in \U^+}P(u)$. Then, $P(\U^-)= 1-\theta$. 
		Since $\sum_{u\in\{0,1\}^n}\alpha_u=0$ and $\alpha_u\in[-1,1]$, for all $u\in\{0,1\}^n$, it is easy to see that $\sum_{u\in\{0,1\}^n}\alpha_u P(u)$ is maximized for the following choice of alpha: 
		\begin{equation}
			\alpha_u=\begin{cases}
				1,& u\in \U^+ \\ -1, & u\in\U^-.
			\end{cases}
		\end{equation}
		Therefore, for any choice of $\alpha_u$,  $u\in\{0,1\}^n$, that satisfy the conditions in the lemma statement, we have
		\begin{equation}\label{eqa:lemalpha 1}
			\begin{split}
				\left( \sum_{u\in\{0,1\}^n} \alpha_u P(u)  \right)^2\, &\le\, \left( \sum_{u\in \U^+} P(u)  - \sum_{u\in \U^-} P(u)\right)^2\\
				&=\, \big(P(\U^+)  - P(\U^-) \big)^2 \\
				&=\, \big(2\theta-1 \big)^2.
			\end{split}
		\end{equation}
		
		%Let $U^+$ and $U^-$ be the restrictions of the random variable $U$ on $\U^+$ and $\U^-$, respectively. We denote the probability distributions of $$ In other words, for any $u\in\U^+$, we have $$
		Since each of $\U^+$ and $\U^-$ has $2^{n-1}$ elements, it follows that
		\begin{equation}
			H\big(U \mid U\in\U^+\big) \le n-1 \quad \textrm{and}\quad H\big(U \mid U\in\U^-\big) \le n-1.
		\end{equation}
		It then follows from the \emph{grouping axiom} (see \cite{Ash90}, page 8) that 
		\begin{equation}\label{eqa:lemalpha 2}
			\begin{split}
				H(U) \,&=\, h(\theta)\, + \, \theta H\big(U \mid U\in\U^+\big) \, +\, (1-\theta)H\big(U \mid U\in\U^-\big)\\
				&\le\,  h(\theta)\,+\, \theta(n-1)\,+\, (1-\theta)\, (n-1)\\
				&=\, h(\theta)\,+\, n-1,
			\end{split}
		\end{equation}
		where $h(\theta)=\theta\log_2(1/\theta)+(1-\theta)\log_2\big(1/(1-\theta)\big)$ is the entropy of a binary random variable that equals $1$ if $U\in\U^+$ and equals $0$ otherwise.
		
		Consider the function $f(x)=(2x-1)^2 + 1.5 h(x)$, defined for $x\in[0,1]$. Then, for any  $x\in(0,1)$,
		\begin{equation}
			f''(x) \,=\, 8 - \frac{1.5}{x\ln2} - \frac{1.5}{(1-x)\ln2}  \, \le\, 8 - 2\left(\frac{1}{x} + \frac{1}{1-x}\right) \,\le\, 0.
		\end{equation}
		Hence, $f$ is a concave function and is symmetric over $[0,1]$. Therefore, $f(x)$ takes its maximum at $x=1/2$. As a result, for any $x\in[0,1]$,
		\begin{equation}\label{eqa:lemalpha 3}
			\big(2x-1\big)^2\, +\, 1.5h(x)\,=\, f(x) \,\le\, f(1/2)\, =\, 1.5.
		\end{equation}
		Combining \eqref{eqa:lemalpha 1}, \eqref{eqa:lemalpha 2}, and \eqref{eqa:lemalpha 3}, we obtain
		\begin{equation}
			\begin{split}
				\left( \sum_{u\in\{0,1\}^n} \alpha_u P(u)  \right)^2\, +\, 1.5H(U)\, &\le\, \big(2\theta-1 \big)^2 \, +\, 1.5H(U)\\
				&\le\, \big(2\theta-1 \big)^2 \, +\, 1.5h(\theta)\,+\, 1.5(n-1)\\
				&\le\, 1.5\,+\, 1.5(n-1)\\
				&=\, 1.5n,
			\end{split}
		\end{equation}
		where the inequalities are respectively due to \eqref{eqa:lemalpha 1}, \eqref{eqa:lemalpha 2}, and \eqref{eqa:lemalpha 3}. This implies \eqref{eqa:alpha 1.5} and completes the proof of Lemma~\ref{lem:alpha}.

		%%%%%%%%%%%%%%%%%%%%%%%%%%%%%%%%%%%%%%%%%%%%%%%%%%%%%%%%%%%%%%%%%%%%%%%%%%%%%%%%%%%%%%%%%%%%%%%%%%%%%
		%%%%%%%%%%%%%%%%%%%%%%%%%%%%%%%%%%%%%%%%%%%%%%%%%%%%%%%%%%%%%%%%%%%%%%%%%%%%%%%%%%%%%%%%%%%%%%
		%%%%%%%%%%%%%%%%%%%%%%%%%%%%%%%%%%%%%%%%%%%%%%%%%%%%%%%%%%%%%%%%%%%%%%%%%5

		\medskip
		
		\section{Proofs of Lemmas for the Centralized Lower Bound Proof in Section~\ref{subsec:pr0}}
		\subsection{Proof of Lemma~\ref{lem:9 coin}}\label{app:proof lem 9 coins}
		For $i=1,\ldots,9$ and $j=1,\ldots, mn$, let  $x_{j}^i\in \{-1,1\}$ be the outcome of $j$th flip of the $i$th coin.
		For $i=1,\ldots,9$, let  $N^i = (x^i_1+1)/2+\cdots+(x^i_{mn}+1)/2$ be the total number of observed $1$s for the $i$th coin. 
		We assume that the index of the biased coin is unknown and has a uniform prior. 
		According to the 
		Neyman-Pearson lemma (see page 59 in \cite{Lehm06}), the \emph{most powerful} test is the the likelihood ratio test that outputs a coin index $\hat{T}=i$ with the maximum value of $N^i$.
		Below, we derive a lower bound on the error probability of the above test, i.e., $\Pr\big(\hat{T}\ne T\big)$. 
		
		Without loss of generality assume that $T=1$.  Then, $\E[x^1_1]=1/2\sqrt{mn}$, $\var(x^1_1) = 1-1/4mn$, and $\E\big[ x^1_1-\E[x^1_1] \big]^3 = 1-1/16m^2n^2$.
		Let
		\begin{equation} \label{eq:def Y iid sum}
			Y^1 \,=\, \frac{\sum_{j=1}^{mn} \big(x^1_j - \E[x^1_j]\big)}{\sqrt{mn \,\var(x^1_1)} },
		\end{equation}
		and for $i=2,\ldots,9$ let $Y^i=x^i_1+\ldots+ x^i_{mn}$. 
		Then, 
		\begin{equation}\label{eq:N1}
			N^1 \,=\, \frac{mn}2 \,+\, \frac{\sqrt{mn}\sqrt{1-1/4mn}\, Y^1 + \sqrt{mn}/2}2,
		\end{equation}
		and for $i=2,\ldots,9$,
		\begin{equation}\label{eq:N2}
			N^i \,=\, \frac{mn}2 \,+\, \frac{\sqrt{mn}\, Y^i}2.
		\end{equation}
		It then follows from the Berry-Esseen theorem (see  \cite{Serf09}, page 33) that for any $i\le9$ and any $t\in\R$,
		\begin{equation}\label{eq:ber ess}
			\big| \Pr\big( Y^i>t \big) -Q(t)  \big| \,\le\, \frac{33}4 \,\frac{\E\big[ x^i_1-\E[x^i_1] \big]^3}{\var(x^i_1)^{1.5}\, \sqrt{mn}},
		\end{equation}
		where $Q(\cdot)$ is the Q-function of the standard normal distribution. 
		Therefore, ,
		\begin{equation} \label{eq:ber ess Y1}
			\begin{split}
				\Pr\left(N^1 > \frac{mn}2  +  0.4 \sqrt{mn}\right) \, &=\, \Pr\left( \frac{\sqrt{mn}\sqrt{1-1/4mn}\, Y^1 + \sqrt{mn}/2}2 >   0.4 \sqrt{mn}\right)\\
				&=\, \Pr\left( Y^1  >   \frac{0.3}{\sqrt{1-1/4mn}} \right)\\
				&\le\, Q\left(  \frac{0.3}{\sqrt{1-1/4mn}} \right) + \frac{33}4 \,\frac{1+1/4mn}{(1-1/4mn)^{1.5} \sqrt{mn}}\\
				&\le \, 0.3961,
			\end{split}
		\end{equation}
		where the first equality is due to \eqref{eq:N1}, the first inequality follows from \eqref{eq:ber ess}, and the last inequality is from the assumption $mn\ge350000$ in \eqref{eq:mn bound}. 
		In the same vein, for $i=2,\ldots,9$,
		\begin{equation} \label{eq:ber ess Y2}
			\begin{split}
				\Pr\left(N^i \le \frac{mn}2  +  0.4 \sqrt{mn}\right) \, &=\, \Pr\left(  \frac{\sqrt{mn}\, Y^i}2 \le   0.4 \sqrt{mn}\right)\\
				&=\, \Pr\left( Y^i  \le   0.8 \right)\\
				&\le\, 1-Q\left(  0.8 \right) + \frac{33}{4 \sqrt{mn}}\\
				&\le \, 0.8021,
			\end{split}
		\end{equation}
		where the first equality is due to \eqref{eq:N2}, the first inequality follows from \eqref{eq:ber ess}, and the last inequality is from the assumption $mn\ge350000$ in \eqref{eq:mn bound}. 
		Consequently,
		\begin{equation}\label{eq:8286}
			\Pr\left(\max\big(N^2,\ldots,N^9\big) >  \frac{mn}2  +  0.4 \sqrt{mn} \right) \,=\, 1- \Pr\left(N^2 \le \frac{mn}2  +  0.4 \sqrt{mn}\right)^8
			\, \ge\, 1- 0.8021^8 \,>\, 0.8286.
		\end{equation}
		Finally, for the error probability of the aforementioned maximum likelihood test, we have
		\begin{equation}
			\begin{split}
				\Pr\big(\hat{T}\ne T\big) &\, =\, \Pr\Big(\max\big(N^2,\ldots,N^9\big) > N^1 \Big)\\
				&\,\ge\, \Pr\left(\max\big(N^2,\ldots,N^9\big) >  \frac{mn}2  +  0.4 \sqrt{mn} \quad \mbox{and}\quad N^1\le \frac{mn}2  +  0.4 \sqrt{mn} \right)\\
				&\,=\, \Pr\left(\max\big(N^2,\ldots,N^9\big) >  \frac{mn}2  +  0.4 \sqrt{mn} \right)\,\times\, \Pr\left( N^1\le \frac{mn}2  +  0.4 \sqrt{mn} \right)\\
				&\,\ge\, 0.8286\,\times\, \Pr\left( N^1\le \frac{mn}2  +  0.4 \sqrt{mn} \right)\\
				&\,\ge\,  0.8286\,\times\,  (1-0.3961)\\
				&\,>\,\frac12,
			\end{split}
		\end{equation}
		where the second equality is due to the independence of different coins, the second inequality follows from \eqref{eq:8286}, and the third inequality is from \eqref{eq:ber ess Y1}.
		This completes the proof of Lemma~\ref{lem:9 coin}.

		\subsection{Proof of Lemma~\ref{lem:centralized}} \label{app:proof centralized bound}
		Consider a function $\tilde{h}:\R^d\to\R$ as follows. For any $\theta\in\R^n$,
		\begin{equation*}
			h(\theta)\,=\, \begin{cases}
				1/2 - \|\theta\|\quad& \textrm{if } \|\theta\| \le 1/2,\\
				0& \textrm{otherwise.}
			\end{cases}
		\end{equation*}
		Let $\tilde\G=\{-1,0,1\}^2$ be the integer grid with 9 points inside $[-1,1]^2$. 
		To any function $\sigma:\tilde\G\to\{-1,1\}$, we associate a  function $\tilde{f}_\sigma(\theta)\,\triangleq\, \sum_{p\in\tilde\G} \sigma(p) \,\tilde{h}(\theta-p)$ for all $\theta\in\R^n$.

		For any $p\in\tilde\G$, we define a  probability distribution $\tilde{P}_p$ over  functions $\tilde{f}_\sigma$ as follows. For any $\sigma:{\tilde\G}\to\{-1,1\}$,
		\begin{equation*}
			\tilde{P}_p(\tilde{f}_\sigma)\, = 2^{-9}\,\left(1-\frac{\sigma(p)}{2\sqrt{mn}}\right).
		\end{equation*}
		Intuitively, when a function $\tilde{f}_\sigma$ is sampled from $\tilde{P}_p$, it is as if for every $q\in\tilde\G$ with $q\ne p$, we have $\Pr\big(\sigma(q)=1\big)=\Pr\big(\sigma(q)=-1\big)=1/2$, and for $q=p$ we have $\Pr\big(\sigma(p)=1\big)=1/2- 1/\big(4\sqrt{mn}\big)$.
		This is like, the values of $\sigma(q)$ for $q\ne p$ are chosen independently at random according to the outcome of a fair coin flip, while the value of $\sigma(p)$ is the outcome of an unfair coin flip with bias $-1/\big(4\sqrt{mn}\big)$.
		Similar to \eqref{eq:F h}, it is easy to show that $F(\theta)=h\big(\theta-p\big)/2\sqrt{mn}$.
		Therefore,  under probability distribution $P_p$, $\theta^*=p$ is the global minimizer of $F(\cdot)$, and for any $\theta\in \R^n$ with $\|\theta-p\|\ge 1/2$, we have $F(\theta)\,\ge\,F(\theta^*) + 1/4\sqrt{mn}$.
		Therefore, if there exists an estimator under which $F(\hat\theta)\,<\,F(\theta^*) + 1/4\sqrt{mn}$, with probability at least $1/2$, then we have $\|\hat\theta-p\|<1/2$ with probability at least $1/2$.
		In this case, $p$ is the closest grid-point of $\G$ to $\hat\theta$, and we can recover $p$ from $\hat\theta$, with probability at least $1/2$.
		This contradict Lemma~\ref{lem:9 coin}. Consequently, under any estimator, we have $F(\hat\theta)\,\ge\,F(\theta^*) + 1/4\sqrt{mn}$, with probability at least $1/2$.
		This completes the proof of Lemma~\ref{lem:centralized}.

		%%%%%%%%%%%%%%%%%%%%%%%%%%%%%%%%%%%%%%%%%%%%%%%%%%%%%%%%%%%%%%%%%%%%%%%%%%%%%%%%%%%%%%%%%

		%%%%%%%%%%%%%%%%%%%%%%%%%%%%%%%

		%%%%%%%%%%%%%%%%%%%%%%%%%%%%%%%%%%%%%%%%%%%%%%%%
		
		\medskip
		\section{Proof of Lemmas for the Upper Bound Proof in Section~\ref{sec:proof main alg c}}
		\subsection{Proof of Lemma~\ref{lem:prob of E2 upper}} \label{app:proof lem E2 upper}
		%\red{Note that if $\delta>1$, then $t=0$ and the lemma would be trivial. Therefore, we assume that $\delta\le 1$.}
		We begin with a simple inequality:
		for any $x\in[0,1]$ and any $k>0$,
		\begin{equation}\label{eq:power exp min}
			1-(1-x)^k \, \ge\, 1- e^{-kx} \, \ge\, \frac12 \min\big(kx,1\big).
		\end{equation}
		
		Let $Q_{p}$ be the  probability that $p$ appears in the $p$-component of at least one of the sub-signals of machine $i$.
		Then, for $p \in G^l$,
		\begin{equation*} \label{eq:Qp and qp for p}
			\begin{split}
				Q_{p} \,&=\, 1-\left(1-2^{-dl}\times\frac{2^{(d-2)l}}{\sum_{j=1}^t 2^{(d-2)j}}\right)^{\lfloor B/d\log_2 mn  \rfloor} \\
				\, &\ge\, \frac12 \min\left(\frac{2^{-2l}\,\big\lfloor B/(d\log_2 mn)  \big\rfloor}{\sum_{j=1}^t 2^{(d-2)j}},\,1\right)\\
				\, &\ge\, \frac12 \min\left(\frac{2^{-2l} B}{2d\ln (mn) \,\sum_{j=1}^t 2^{(d-2)j}},\,1\right),
			\end{split}
		\end{equation*}
		where the equality is due to the probability of a point $p$ in   $G^l$ (see \eqref{eq:prob choose p c})  and the number  $\lfloor B/(d\log_2 mn)  \rfloor$  of sub-signals per machine, and the first inequality is due to \eqref{eq:power exp min}. Then, 
		\begin{equation}
			\begin{split}
				\mathbb{E}\big[N_{p}\big]\,&=\,  Q_{p} m
				\,\geq\, \min\left(\frac{2^{-2l} m B}{4d\ln (mn) \, \sum_{j=1}^t 2^{(d-2)j}},\, \frac{m}{2}\right).
			\end{split}
			\label{eq:8star 1}
		\end{equation}
		%where the last equality follows from \eqref{eq:Qp and qp n} and $\mathcal{E'}$.

		We now  bound the two terms on the right hand side of \eqref{eq:8star 1}.
		For the second term on the right hand side of \eqref{eq:8star 1}, we have
		\begin{equation}
			\begin{split}
				\frac{m}{2} \,& =\, \frac{m\epsilon^2}{2\epsilon^2} \\
				\,&\ge\,\frac{16 md\,\ln^4 mn}{2 mn \epsilon^2} \\
				\,&=\, \frac{8d\,\ln^4 mn}{ n \epsilon^2},
			\end{split}
			\label{eq:8star 3}	
		\end{equation}
		where the first inequality is from the definition of $\epsilon$ in \eqref{eq:def eps upper}.
		For the first term at the right hand side of \eqref{eq:8star 1}, note that
		\begin{equation}\label{eq:1delta vs m}
			t =\log_2(1/\delta) \,\le\, \log_2\left(\frac{\sqrt{m}}{\ln mn}\right) \,<\, \ln m.
		\end{equation}
		It follows that for any $d\ge 1$,
		\begin{equation*}\label{eq:t 2tmax bound}
			\begin{split}
				\sum_{j=1}^t 2^{(d-2)j} \,&\le\, t  2^{t(d-2)}\\ \,&\le\, \ln(mn)\, 2^{t(d-2)}\\
				&=\, \ln(mn)\, \left(\frac{1}\delta\right)^{(d-2)}\\
				&=\, \ln(mn)\,\delta^2\, \left(\frac{1}\delta\right)^{d}\\
				&\le\, \ln(mn)\,\delta^2\, \frac{mB}{\ln^{2d} mn}\\
				&=\, \ln(mn)\times \frac{n\epsilon^2}{16d\ln^2{mn}} \times \frac{mB}{\ln^{2d} mn}\\
				&\le\, \frac{nmB\epsilon^2}{16d \ln^5 mn},
			\end{split}
		\end{equation*}
		%where the first equality is by definition of $t$, and 
		where the second inequality is due to \eqref{eq:1delta vs m}, the third inequality follows from the definition of $\delta$, the third equality is from the definition of $\epsilon$ in \eqref{eq:def eps upper}, and the last inequality is because of the assumption $d\ge2$.
		Then,
		\begin{equation} \label{eq:mb over n sigma bound c}
			\begin{split}
				\frac{2^{-2l} m B}{4d\ln (mn) \, \sum_{j=1}^t 2^{(d-2)j}}\,&\ge\, 
				\frac{2^{-2l} m B}{4d\ln (mn) } \times \frac{16d \ln^5 mn}{nmB\epsilon^2}\\
				&=\,\frac{4\ln^4(mn)\, 2^{-2l}}{n\epsilon^2}.
			\end{split}
		\end{equation}
		Consequently,
		\begin{equation} \label{eq:mb over n sigma bound c 22}
			\begin{split}
				\frac{2^{-2l} m B}{4d\ln (mn) \, \sum_{j=1}^t 2^{(d-2)j}}\,&\ge\, 
				\frac{4\ln^4(mn)\, 2^{-2l}}{n\epsilon^2}\\
				&\, \ge\, \frac{4\ln^4(mn)\, 2^{-2t}}{n\epsilon^2} \\
				&\, =\, \frac{4\ln^4(mn)\, \delta^2}{n\epsilon^2} \\
				&\, =\, \frac{4\ln^4(mn)\, \delta^2}{16d \delta^2 \,\ln^2 mn} \\
				&\,=\, \frac{\ln^2 (mn)}{4d}, 
			\end{split}
		\end{equation}
		where the first equality is due to the definition of $t=\ln_2 (1/\delta)$, and the second equality is from the definition of $\epsilon$.
		Plugging \eqref{eq:8star 3} and \eqref{eq:mb over n sigma bound c} into \eqref{eq:8star 1}, it follows that for  $l=1,\ldots,t$ and for any $p\in {G}^l$,
		\begin{equation}\label{eq:8stat s}
			\mathbb{E}\big[N_{p}\big] \,\ge\, \frac{4\ln^4(mn)\, 2^{-2l}}{n\epsilon^2}.
		\end{equation}	
		Moreover, plugging   \eqref{eq:mb over n sigma bound c 22}  into \eqref{eq:8star 1}, we obtain
		\begin{equation}\label{eq:frac to log 2}
			\begin{split}
				\frac18 \mathbb{E}\big[N_{p}\big] \,&\ge\, \frac18\,\min\left(  \frac{\ln^2 (mn)}{4d} ,\, \frac{m}2\right) \\
				&\, \ge\, \frac18\, \min\left(  \frac{\ln^2 (mn)}{4d} ,\, \frac{\ln^2 mn}2\right) \\
				&\, \ge\,\frac{\ln^2 (mn)}{32d},
			\end{split}
		\end{equation}
		where the second inequality is because of the assumption $m\ge \ln^2 mn$ in \eqref{eq:assum1}.
		Then, for  $l\in 1,\ldots,t$ and any $p\in \tilde{G}_{s^*}^l$,
		\begin{equation}\label{eq:pr NP log4}
			\begin{split}
				\Pr\left(N_p\leq \frac{2\ln^4(mn)\, 2^{-2l}}{n\epsilon^2}\right) \,&\leq\, \Pr\left(N_p\leq \frac{\mathbb{E}[N_p]}{2} \right)\\
				\,& \leq \,\exp\left(-(1/2)^2\mathbb{E}[N_p]/2\right)\\
				\,&\le\,\exp\big({-\ln^2(mn)/32d}\big),
			\end{split}
		\end{equation}
		where the inequalities are due to \eqref{eq:8stat s}, Lemma~\ref{lemma:CI}~(b), and \eqref{eq:frac to log 2}, respectively.
		Then,
		\begin{align*}
			\Pr\big(\mathcal{E} \big)\, 
			&=\, \Pr\left(N_p\geq  \frac{2\ln^4(mn)\, 2^{-2l}}{n\epsilon^2}   , \quad  \forall p\in G^l\mbox{ and for }l=1,\ldots,t\right)\\
			&\ge\,
			1- \sum_{l=1}^{t} \sum_{p\in G^l} \Pr\left(N_p< \frac{2\ln^4(mn)\, 2^{-2l}}{n\epsilon^2}\right)\\
			&\geq\, 1- t2^{dt} \exp\big(-\ln^2(mn)/(32d)\big)\\
			&=\, 1- \ln(1/\delta)\left(\frac1\delta\right)^d \exp\big(-\ln^2(mn)/(32d)\big)\\
			&\ge\, 1- {\ln(mn)}\,\frac{m^{d/2}}{\ln^d mn} \exp\big(-\ln^2(mn)/32d\big)\\
			&\ge\, 1- m^{d/2} \exp\big(-\ln^2(mn)/32d\big),
		\end{align*}
		where the first equality is by the definition of $\mathcal{E}$, the first inequality is from union bound, the second inequality is due to \eqref{eq:pr NP log4}, and the third inequality follows from \eqref{eq:1delta vs m} and the definition of $\delta$ in \eqref{eq:def delta c}.
		This completes the proof of Lemma~\ref{lem:prob of E2 upper}.

		%%%%%%%%%%%%%%%%%%%%%%%%%%%%%

		\medskip
		\subsection{Proof of Lemma~\ref{lem:prob E3 upper}} \label{app:proof E3 upper}
		For any $l\le t$ and any $p\in G^l$, let
		$$\hat{\Delta}(p)= \frac{1}{N_p}\,\, \sum_{\substack{\mbox{\scriptsize{Subsignals of the form }} \\ (p,\Delta,\cdot,\cdot)\\ \mbox{\scriptsize{after redundancy elimination}}}} \Delta,$$
		and let $\Delta^*(p)= \mathbb{E}[\hat{\Delta}(p)]$.
		
		%\red{[Removed the computations for the case of $l=0$.]}
		\hide{
			We first consider the case of $l=0$. Note that $G^0$ consists of a single point $p=s^*$.
			Moreover, the component $\Delta$ in each signal is the average over the gradient of $n/2$ independent functions. 
			Then, $\hat{\Delta}(p)$ is the average over the gradient of $N_p \times n/2$ independent functions.
			%In level 0, we have only one point $p=s^*$. Let \linebreak $\hat{\Delta}:= \frac{1}{N_p}\sum_{y^i\mbox{\scriptsize{'s such that }} (s^*,p,\Delta)}\Delta$ and $\Delta^*:= \mathbb{E}[\hat{\Delta}]$. 
			Given event $\mathcal{E}$, it follows from Hoeffding's inequality (Lemma~\ref{lemma:CI}~(a)) that
			\begin{equation}
				\begin{split}
					\Pr&\left(\big|\hat{F}(s^*)-F(s^*)\big|\geq \frac{\epsilon}{8\ln(mn)}\right)\\
					&\qquad\leq\, \exp\left(-N_{s^*}n \times \left(\frac{\epsilon}{8\ln(mn)}\right)^2  \right)\\
					&\qquad\leq \exp\left(-n\times \frac{\ln^{2d-6}(mn)(mB)^{2/d}}{4d} \times\frac{\epsilon^2}{64\, \ln^2(mn)}\right)\\
					&\qquad=\exp\left(\frac{-\ln^2(mn)}{512d}\right)\\
					&\qquad=\, \exp\Big(-\Omega\big(\ln^2(mn)\big)\Big).
				\end{split}
				\label{eq:10star}
			\end{equation}
		}
		%TODO: dividing by 2?
		
		For $l\geq 1$, consider a grid point $p\in G^l$ and let $p'$ be the parent of $p$.
		Then, $\|p-p'\|=\sqrt{d}\,2^{-l}$.
		%$\|p-p'\|=\frac{\sqrt{d}\ln(mn)}{2\sqrt{n}}2^{-l}$. 
		Furthermore, by definition, for any function $f\in \F$, we have $|f(p)-f(p')|\leq \|p-p'\|$.
		Therefore, $\hat{\Delta}(p)$ is the average of $N_p \times n/2$ independent variables with absolute values no larger than $\sqrt{d}\, 2^{-l}$. 
		Given event $\mathcal{E}$, it then follows from the Hoeffding's inequality that
		\begin{align*}
			\Pr&\left(\big|\hat{\Delta}(p)-\Delta^*(p)\big|
			\geq\, \frac{\epsilon}{8\ln(mn)}\right)\\
			&\leq\, 2\exp\left(-{nN_p}\times\frac1{(2\sqrt{d}\,2^{-l})^2}\times \left(\frac{\epsilon}{8\ln mn}\right)^2\right)\\
			&\leq\, 2\exp\left(-n \times \frac{2\ln^4(mn)\, 2^{-2l}}{n\epsilon^2}\times \frac{1}{4d\,2^{-2l}}\times \frac{\epsilon^2}{64 \ln^2 mn}\right)\\
			&=\,2\exp\big(-\ln^2(mn)/128d\big),
		\end{align*}
		
		Recall from \eqref{eq:page17 c} that for $l=1,\ldots, t$ and any $p \in {G}^l$ with parent $p'$,
		$$\hat{F} (p)- F(p)\,=\,\hat{F} (p')-F(p')+\hat{\Delta}(p)-\Delta^*(p).$$
		Then,
		\begin{align*}
			\Pr&\left(|\hat{F} (p)-F(p)|>\frac{l\epsilon}{8\ln mn}\right)\\
			&\leq\,\Pr\left(|\hat{F} (p')-F(p')|>\frac{(l-1)\epsilon}{8\ln mn}\right)
			\, +\,\Pr\left(|\hat{\Delta}(p)-\Delta^*(p)|>\frac{\epsilon}{8\ln mn}\right)\\
			&\leq\,\Pr\left(|\hat{F} (p')-F(p')|>\frac{(l-1)\epsilon}{8\ln mn}\right) +2\exp\big(-\ln^2(mn)/128d\big).
		\end{align*}
		Employing an induction on $l$, we obtain for any $l\le t$ and any $p\in G^l$,
		$$\Pr\left(|\hat{F} (p)-F(p)|>\frac{l\epsilon}{8\ln mn}\right)
		\,\le\,2  l \,\exp\big(-\ln^2(mn)/128d\big).$$
		Therefore, % for any grid point $p$, 
		\begin{equation}\label{eq:this}
			\begin{split}
				\Pr\left(|\hat{F} (p)- F(p)|>\frac{\epsilon}{8}\right)\,
				&\le\,\Pr\left(|\hat{F} (p)-F(p)|>\frac{l\epsilon}{8\ln mn}\right)\\
				%&\le\,\exp\Big(-\Omega\big(\ln^2(mn)\big)\Big).\\
				&\le\,2  \ln(m) \,\exp\big(-\ln^2(mn)/128d\big),
			\end{split}
		\end{equation} 
		where the inequalities are due to \eqref{eq:1delta vs m}.
		It then follows from the union bound that
		\begin{equation} \label{eq:E''' given E''}
			\begin{split}
				\Pr\big(\mathcal{E'} \mid \mathcal{E}\big)\, &\geq\,  1- \sum_{l=1}^{t} \sum_{p\in {G}^l}\Pr\left(|\hat{F} (p)-F(p)|>\frac{\epsilon}{8}\right)\\
				&\geq\, 1- t2^{dt} \times 2  \ln(m) \,\exp\big(-\ln^2(mn)/128d\big)\\
				&\ge\, 1- \ln(m) \times \left(\frac1\delta\right)^d \times 2  \ln(m) \,\exp\big(-\ln^2(mn)/128d\big)\\
				&\ge\, 1- \ln(m) \times \frac{m^{d/2}}{\ln^d mn}\times 2  \ln(m) \,\exp\big(-\ln^2(mn)/128d\big)\\
				&\ge \,1-2  m^{d/2} \,\exp\big(-\ln^2(mn)/128d\big),
			\end{split}
		\end{equation} 
		where the second inequality is due to \eqref{eq:this}, the third inequality follows from \eqref{eq:1delta vs m}, and the fourth inequality is from the definition of $\delta$.
		On the other hand, we have from Lemma~\ref{lem:prob of E2 upper} that $\Pr\big(\mathcal{E} \big)= 1-m^{d/2} \exp\big(-\ln^2(mn)/8d\big)$.
		Then, $\Pr\big(\mathcal{E'} \big)\ge 1-m^{d/2} \exp\big(-\ln^2(mn)/32d\big)-2m^{d/2} \exp\big(-\ln^2(mn)/128d\big)$ and Lemma~\ref{lem:prob E3 upper} follows.

		%%%%%%%%%%%%%%%%%%%%%%%%%%%%%%%%%%%%%%%%%%%%%%%%%%%%%%%%%%%%%%%%%%%
		%%%%%%%%%%%%%%%%%%%%%%%%%%%%%%%%%%%%%%%%%%%%%%%%%%%%%%%%%%%%%%%%5
		%%%%%%%%%%%%%%%%%%%%%%%%%%%
		%%%%%%%%%%%%%%%%%%%%%%%%%%%%%%%%%%%%%%%%%%%%%%%%%%%%%%%%%%%%%%%%
		
		%%%%%%%%%%%%%%%%%%%%%%%%%%%%%

		\medskip
		\subsection{Proof of Lemma~\ref{lem:prob E4 upper}} \label{app:proof E4 upper}
		Fix a machine $i$ and let $g(\theta)=(F^i(\theta)-F^i(p))-(F(\theta)-F(p))$, for all $\theta\in [-1,1]^d$.
		Note that for any function $f\in \mathcal{F}$, any $p\in G^t$ and any $\theta\in cell_p$, we have $|f(\theta)-f(p)|\leq \|\theta-p\|\leq \sqrt{d} \delta$. Then, $F^i(\theta)-F^i(p)$ is the average over $n/2$ randomly chosen such functions $f(\theta)-f(p)$ with the expected value $F(\theta)-F(p)$. It follows from Hoeffding's inequality (Lemma \ref{lemma:CI}) that:
		\begin{equation}\label{eq:g}
			\begin{split}
				\Pr\Big(\big|g(\theta)\big|>\frac{\epsilon}{16}\Big) &= \Pr\Big(\big|(F^i(\theta)-F^i(p))-(F(\theta)-F(p))\big|>\frac{\epsilon}{16}\Big) \\
				&\leq 2\exp\Bigg(-\frac{2\times n/2\times(\epsilon/16)^2}{(2\sqrt{d}\delta)^2}\Bigg)\\
				&=2\exp\Bigg(-n\Bigg(\frac{4\delta \sqrt{d}\ln(mn)}{16\sqrt{n}\times 2\sqrt{d}\delta}\Bigg)^2\Bigg)\\
				&=2\exp\Bigg(-\frac{\ln^2(mn)}{64}\Bigg),
			\end{split}
		\end{equation} 
		where the first equality is due to the definition of $\epsilon=\delta \sqrt{d}\ln(mn)/\sqrt{2n}$. 
		
		Consider  a regular grid $\mathcal{D}$ with edge size $\epsilon/16\sqrt{d}$ over $cell_p$. Then, 
		$$|\mathcal{D}| =\left(\frac{2\delta}{\epsilon/16\sqrt{d}}\right)^d = \left(\frac{32\delta \sqrt{d} \sqrt{n}}{4\delta \sqrt{d} \ln mn}\right)^d = \left(\frac{8 \sqrt{n}}{\ln mn}\right)^d \le n^{d/2},$$
		where the second inequality is due to the definition of $\epsilon$, and the last inequality is due to the assumption $\ln mn \ge 8\sqrt{d}$ in \eqref{eq:assum1}
		It then follows from \eqref{eq:g} and the union bound that with probability at least $1-2  n^{d/2} \exp\big(-\ln^2(mn)/64\big)$, we have 
		\begin{equation}\label{eq:**}
			|g(\theta)| \le  \frac{\epsilon}{16}, \quad \forall \theta \in \mathcal{D}.
		\end{equation}
		On the other hand the function $g(\theta)=(F^i(\theta)-F^i(p))-(F(\theta)-F(p))$ is the sum of two Lipschitz continuous functions, and is therefore  Lipschitz continuous with constant 2. 
		Consider an arbitrary $\theta\in cell_p$ and let $\theta'$ be the closest grid point in $\mathcal{D}$ to $\theta$. Then, $\|\theta-\theta'\| \le \epsilon/32$. 
		Then, assuming \eqref{eq:**}, we have
		\begin{equation}\label{eq:***}
			\begin{split}
				|g(\theta)| \,&\le\, |g(\theta)| + \big|g(\theta') - g(\theta)\big|\\
				&\le\,\frac{\epsilon}{16} + \big|g(\theta') - g(\theta)\big|\\
				&\le\,\frac{\epsilon}{16} + 2 \big\|\theta'-\theta\big\|\\
				&\le\,\frac{\epsilon}{16} + \frac{2\epsilon}{32}\\
				&=\,\frac{\epsilon}{8},
			\end{split}
		\end{equation}
		where the second inequality is due to \eqref{eq:**} and the third inequality follows from the Lipschitz continuity of $g$ with constant 2. 
		Employing union bound over all machines  $i$  and all cells $cell_p$ for $p\in G^t$,
		it follows from \eqref{eq:**} and \eqref{eq:***} that $\mathcal{E}''$ holds true with probability at least $1-2  n^{d/2} m^{1+d/2}\exp\big(-\ln^2(mn)/64\big)$.
		This completes the proof of Lemma~\ref{lem:prob E4 upper}.

		\ifCLASSOPTIONcaptionsoff
		\newpage
		\fi

	\end{document}